\documentclass[review]{elsarticle}

\usepackage{comment}
\usepackage{graphicx}
\usepackage{caption}
\usepackage{subcaption}
\usepackage{multirow}
\usepackage{amsmath}
\usepackage{amssymb}
\usepackage{amsthm}
\newtheorem{theorem}{Theorem}
\newtheorem{definition}{Definition}
\newtheorem{lemma}{Lemma}

\begin{document}

\begin{frontmatter}

\title{A convex method for classification of groups of examples}


\author{Dori Peleg \textsuperscript{a} \fnref{myfootnote} \corref{mycorrespondingauthor}}
\address{\textsuperscript{a} Medtronic, Hermon building, Yokneam, Israel 2069204}
\fntext[myfootnote]{Email Address: dori.peleg@medtronic.com}





\begin{abstract}
There are many applications where it important to perform well on a set of examples as opposed to individual examples. For example in image or video classification the question is does an object appear somewhere in the image or video while there are several candidates of the object per image or video. In this context, it is not important what is the performance per candidate. Instead the performance per group is the ultimate objective. 

For such problems one popular approach assumes weak supervision where labels exist for the entire group and then multiple instance learning is utilized. Another approach is to optimize per candidate, assuming each candidate is labeled, in the belief that this will achieve good performance per group. 

We will show that better results can be achieved if we offer a new methodology which synthesizes the aforementioned approaches and directly optimizes for the final optimization objective while consisting of a convex optimization problem which solves the global optimization problem. 
The benefit of grouping examples is demonstrated on an image classification task for detecting polyps in images from capsule endoscopy of the colon. The algorithm was designed to efficiently handle hundreds of millions of examples. Furthermore, modifications to the penalty function of the standard SVM algorithm, have proven to significantly improve performance in our test case.
\end{abstract}

\begin{keyword}
Multiple instance learning, support vector machine, image classification, object recognition, convex optimization, deep learning
\end{keyword}

\end{frontmatter}


\section{Introduction}

There are many applications where there is a meaning to groups of examples. In image classification the question is if there is an object somewhere in an image\footnote{This is different from object detection where the objective is to detect each instance of the object.}. The examples may be candidates of the object and the group is all the candidates in an image. In this context, it is enough for one candidate to be categorized as the desired object and the decision per image is that it contains the object. Only if none of the candidates don't contain the object, then it is decided that the image does not contain the object. The example can be expanded from a single image to a video. For example, does this video contain a certain object (e.g. a specific person).

These types of problems can be dealt in three possible approaches. The third is the one proposed in this paper and is a synthesis of the first two.

\begin{enumerate}
\item Multiple Instance Learning (MIL)- Optimization is on the group level, with only group labels.
\item Optimization is on the candidate level, with candidate labels.
\item Optimization is on the group level, with candidate labels.
\end{enumerate}

The setup of targeting the group level has been the subject of extensive research in MIL (\cite{dietterich1997solving}, \cite{maron1998framework}, \cite{wang2000solving}, \cite{gartner2002multi}, \cite{zhang2002dd}, \cite{andrews2003support}, \cite{weidmann2003two}, \cite{xu2003statistical}, \cite{frank2003applying}, \cite{auer2004boosting}, \cite{zhou2004multi}, \cite{chen2006miles}, \cite{nowak2006sampling}, \cite{kriegel2006approach}, \cite{zhou2007solving}, \cite{kwok2007marginalized}, \cite{bunescu2007multiple}, \cite{zhou2007multi}, \cite{mangasarian2008multiple}, \cite{zhang2009multi}, \cite{zhou2009multi}, \cite{li2009convex}, \cite{vezhnevets2010towards}, \cite{foulds2010review}, \cite{amores2013multiple}, \cite{ray2014multiple}, \cite{herrera2016multiple}, \cite{herrera2016multiple}). In the MIL framework the assumption is that only weak supervision is available. Namely, the label of each group (termed as bag in MIL) is available and the labels of each candidate (termed as instance in MIL) is unknown. The advantage of such an assumption is that one can utilize data with very little supervision effort. For example in images, one needs to determine if the object is in the image or not, without specifying where.

The second approach is to attain labeled candidates and use standard classification algorithms while ignoring the grouping action. It is the premise of this paper that if the ultimate objective function pertains to groups, then directly optimizing on the group level provides superior performance compared to optimizing on the candidate level and then grouping the results. This approach will be compared to the rest empirically in Section \ref{Section:Experiments}.

The third approach is \textbf{not} a subset of the MIL framework since it requires more information. MIL algorithms attempt to solve non-convex optimization problems aimed to guess which candidates (termed instances) of the positive groups (termed as bags) are positive. In this work this information is available and is used to solve a convex optimization problem. Thus inherently the performance of the proposed approach will be at least as good as any MIL algorithm. Thus a comparison to any MIL algorithm is not fair for them.

The question that one must ask for each dataset, where the group answer is important, is how much better can the performance increase with strong supervision and what is the annotation cost. 

For detecting objects in images, in the 'positive' images which contain at least one example of the desired object, one can mark with a trivial GUI either a bounding box or the exact perimeter of the desired object . A candidate is considered 'positive' if it overlaps enough with the mask provided by the human annotator. For the 'negative' images, no annotation is needed.

Such an annotation provides candidate level labels for any algorithm automatically and without a need to repeat human labor for different algorithms or parameters. The annotation typically requires only a few seconds per image. For datasets with relatively few positive examples the annotation cost is negligible and for large scale datasets mechanical turk or other low cost labor can be used.

The test case for this paper will be image classification. In image classification the methodology with non end-to-end learning techniques consists of four steps:
\begin{enumerate}
	\item Region proposal- A group of candidate masks of the desired object are calculated. The purpose of this phase is to have at least one mask which fits well with the desired object, if it appears in the image. This can be at the cost of many candidates which do not mark the desired object.
	\item Feature generation- Features are calculated to help differentiate between the desired object and the rest.
	\item Classifier- A classifier is trained based on the features of the previous step to provide a score per candidate which differentiates between candidates of the desired object and the rest.
	\item Score per image- The probability the image contains the desired object is based on the maximal score of the candidates of the image.

\end{enumerate}

With deep learning (\cite{bengio2009learning}, \cite{hinton2006fast}, \cite{goodfellow2016deep}) steps 2 and 3 are done together and with R-CNN (\cite{girshick2014rich}), fast R-CNN (\cite{girshick2015fast}) and faster R-CNN (\cite{ren2015faster}), all the steps are done together. 

Note that detecting objects that may consist of a very small portion of the image is difficult for such techniques since the majority of the image is of the same distribution in both images with the desired object and without. Typically some cropping should be done and a grouping per image of the results should be done. This brings deep learning techniques back into the setting proposed by this paper.

Deep learning has cemented itself as the state-of-the-art and first choice in visual recognition tasks. However in settings with either very few examples overall (the order of tens-hundreds) or unbalanced data where one class has very few examples, then deep learning may not be the only choice. 

The focus of this paper is on binary classification problems, but it can easily be expanded to multi-class problems similarly to the techniques used for the SVM algorithm (\cite{boser1992training}, \cite{andrew2000introduction}).

Section \ref{Section:Algorithm} gradually introduces the proposed algorithm starting from the primal form of the SVM algorithm. In Section \ref{Section:Experiments} we present a set of comparisons with the SVM algorithm for the challenge of polyp detection in capsule endoscopy. To finalize, Section \ref{Section:Summary} describes the main conclusions and outlines future work.

\section{Algorithm}\label{Section:Algorithm}

The starting point of the development of the algorithm was the primal form of the linear SVM algorithm. The following modifications were performed:
\begin{enumerate}
	\item Balancing positive and negative examples.	
  \item Replacing the hinge-loss function with a smoothed function for the training errors.
  \item Replacing the squared norm 2 with the Huber penalty for the weights of the classifier.
  \item Grouping examples.	
\end{enumerate}

Steps $1$-$3$ are not novel and should be considered as a suggested best practice which contribute to the performance in a problem with very few positive examples and virtually an unlimited number of negative examples. The novelty of the paper is in step $4$ where we suggest a different way to optimize the grouping problem which is different from the MI-SVM \cite{andrews2003support} algorithm and may provide a significant increase in performance.

\subsection{Original SVM classifier}
The standard linear SVM (primal) optimization problem is:

\begin{equation}
\begin{array}{ll}\label{opt_prob:primal_original}
\mathrm{minimize} & w^Tw+C\mathbf{1}^T \xi \\
\mbox{subject to} & \xi \geq \mathbf{1} - Y(X w + b\mathbf{1}) \\
                  & \xi \geq 0,                                                                                   
\end{array}
\end{equation}

where the variables are $w\in \mathbb{R}^{d}$, $b\in \mathbb{R}$, $\xi\in \mathbb{R}^{n}$, $\mathbf{1}$ is a column vector of ones with appropriate dimensions, $X\in\mathbb{R}^{n\times d}$ is the feature matrix and $Y\in\mathbb{R}^{n\times n}$ is a diagonal matrix with the training labels on the diagonal and zero outside of the diagonal. The number of training examples is $n$ and the number of features is $d$.

This is a Quadratic Programing (QP) optimization problem (\cite{boyd2004convex}). The objective function has two terms. The first is the penalty on the squared norm of $w$. The second is the hinge loss on the soft-margins multiplied by the labels. The hinge loss of variable $t$, termed $h(t)$ is the maximum of $1-t$ and $0$. In Figure \ref{fig:hinge}, the hinge loss is depicted.
\begin{figure}[h!]
\centering
\includegraphics[width = 0.7\textwidth]{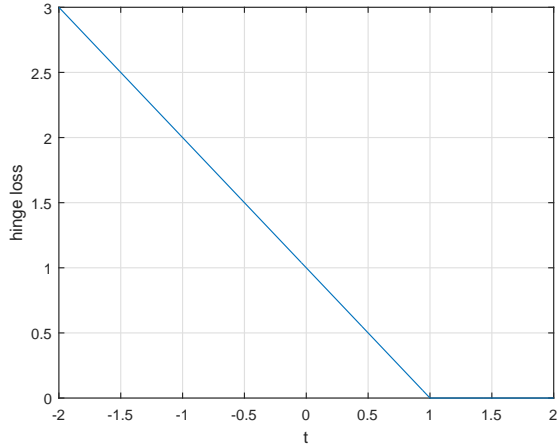}
\caption{Hinge loss function. For a training example with a feature vector $x^i$ and label $y^i$, the input to the hinge loss is $t=y^i (w^T x^i+b)$.}
\label{fig:hinge}
\end{figure}

\subsection{Balancing positive and negative examples}

In highly unbalanced datasets, it is critical not to provide equal weights to both classes as is done in Problem \ref{opt_prob:primal_original}. If this step is not performed, the classifier will typically predict the class with the highest a priori probability.
In order to normalize Problem \ref{opt_prob:primal_original}, so that the search for the optimal hyperparameter will be independent of the number of examples and the number of features, and in order to balance between the number of positive $n_+$ and negative examples $n_-$, the following normalized SVM problem is defined.

\begin{equation}
\begin{array}{ll}\label{opt_prob:primal_normalized}
\mathrm{minimize} & \frac{1-\lambda}{d} w^T w+\frac{\lambda}{n_+}\mathbf{1}^T \xi_+\frac{\lambda}{n_-}\mathbf{1}^T \xi_- \\
\mbox{subject to} & \xi \geq \mathbf{1} - Y(X w + b\mathbf{1}) \\
                  & \xi \geq 0,                                                                                   
\end{array}
\end{equation}

where $\xi_+\in \mathbb{R}^{n_+}$ and and $\xi_-\in \mathbb{R}^{n_-}$ are the vectors of slack variables of the positive and negatives examples respectively and $0\leq \lambda \leq 1$.

\subsection{Shifting from constrained optimization to unconstrained}

The SVM algorithm consists typically of solving the dual optimization problem to Problem \ref{opt_prob:primal_original}. This is done since typically a non-linear classifier provides the best performance for a classification problem. Note that \cite{chapelle2007training} pointed out that thanks to representer theorem (\cite{smola1998learning}) one can use the primal formulation for a non-linear kernel classifier. Furthermore \cite{chapelle2007training} explained why solving the primal formulation converges faster to the solution than the dual.

However, whether one solves the dual optimization problem or uses the primal with a kernel classifier, the number of variables, $n+1$, is related to the number of examples. In the setting where there is a large number of training examples $n$, this becomes problematic not only due to the computational complexity, but even more so due to memory constraints. These approaches consist of utilizing the kernel matrix which is $[n \times n]$. This typically limits the number of training examples to the order of thousands.

In this paper we would like to utilize a large number of examples instead of selecting a subgroup which is aimed at representing the entire training database. In Section \ref{Subsection: Results} the importance of utilizing a large number of examples is demonstrated.

For problems where $n$ or just one of $n_+$ or $n_-$ is extremely large, the number of variables to optimize in the formulation with constraints and slack variables is extremely large. Thus there is no option to add slack variables and we need to solve an optimization problem without constraints. 

Previous work has introduced several variations on the SVM optimization problem which involves solving an unconstrained optimization problem (\cite{FungMangasarian:04}, \cite{Mangasarian:06}, \cite{chapelle2007training}, \cite{franc2008optimized}). Similarly, we propose to optimize the hinge loss directly, without slack variables as detailed in Problem \ref{opt_prob:primal_hinge}.

\begin{equation}
\begin{array}{ll}\label{opt_prob:primal_hinge}
\mathrm{minimize} & \frac{1-\lambda}{d} w^T w+\frac{\lambda}{n_+}\sum_{i\in i_+}{h(y^i (w^T x^i+b))} +\frac{\lambda}{n_-}\sum_{i\in i_-}{h(y^i (w^T x^i+b))} \\
,                                                                                   
\end{array}
\end{equation}

where $i_+$ and $i_-$ are the indices of the positive and negative training examples respectively. Thus the number of variables is $d+1$ and the only dependency on $n$ is in the number of elements in the summation term.

\subsection{The Huber penalty function}

A variant of the primal SVM algorithm (Problem \ref{opt_prob:primal_original}) is using norm $1$ of the variable $w$ instead of the norm $2$ squared. The advantage of using the norm $1$ is that the solution is typically sparse and it is less influenced by outliers \cite{boyd2004convex}.
On the other hand, if all the features are relevant, typically the squared norm $2$ provides better results.

A compromise between the two is the Huber function.

\begin{definition}\label{def: huber}
The Huber cost function
$h_\epsilon:\mathbb{R}\rightarrow\mathbb{R}_+$ is  defined as,

\[ \begin{array}{cc}
   h_\epsilon
	(t)= \left\{\begin{array}{ll} 
	              \frac{t^2}{2\epsilon}  & |t|\leq \epsilon  \\ 
								|t|-\frac{\epsilon}{2} & |t|> \epsilon 
	 \end{array}\right. \\
\end{array} \]
\end{definition}

The Huber function is a continuous and differentiable function. It is depicted in Figure \ref{fig:Huber}.

\begin{figure}[h!]
\centering
\includegraphics[width = 0.7\textwidth]{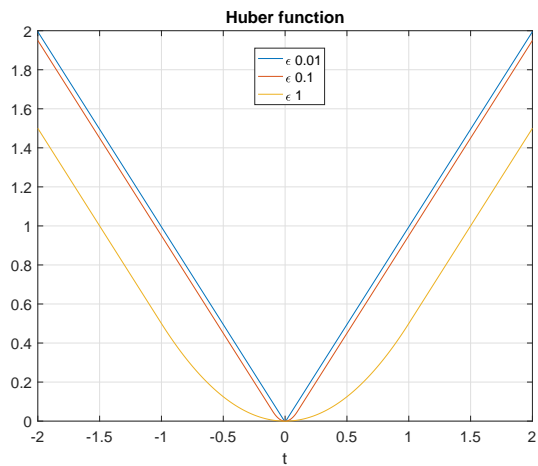}
\caption{Huber cost function, for several values of $\epsilon$. It is equal to the squared norm 2 for $\epsilon\rightarrow\inf$ and norm 1 for $\epsilon=0$.}
\label{fig:Huber}
\end{figure}

On the one hand, it doesn't 'force' the values to be exactly $0$ in the $\epsilon$ area around $0$. On the other hand it is less influenced by outliers as the squared norm $2$. Furthermore, if the norm $1$ penalty should be optimized, for small $\epsilon$, the Huber function is a differentiable approximation of the norm $1$ penalty.

Due to the flexibility of the Huber cost function (two penalty functions, norm 1 and squared norm 2, and all variants in between), it has been used in this algorithm instead of the squared norm 2 of $w$ in Problem \ref{opt_prob:primal_hinge}.

\subsection{Smoothed hinge-loss function}

The hinge-loss is a popular penalty function for training errors. However it has two shortcomings:
\begin{enumerate}
	\item It is not a differentiable function.
	\item It may be providing high penalties for examples which are not training errors and which are in the margin.
\end{enumerate}

The first point is not critical, since sub-gradients can be utilized. However it is preferable to use a differentiable alternative if this doesn't deteriorate performance.

The second point may seem problematic since typically a cross validation process is performed which multiplies the hinge-loss by a positive scalar hyper-parameter $C$ which in effect changes the slope of the hinge-loss and selects the best trade-off between training error and generalization for each dataset (see Problem \ref{opt_prob:primal_original}).

However this does not change the ratio of the penalty of a training error and an example within the margin. For example, the penalty for a training error with a soft-margin of $-\frac{1}{2}$ is equivalent to $3$ examples with a soft-margin of $\frac{1}{2}$ within the margin and which are not training errors. Note that this ratio is not affected by the hyper-parameter $C$. In order to modify this ratio the slope of the loss function must be different in the margin without training error (the $[0,1]$ region) compared to the training error region (the negative orthant).

The loss function in Definition \ref{def: smoothed hinge loss} answers this need. Note that for example if $\delta$ is equal to $\frac{1}{2}$, the aforementioned ratio is $8$ instead of $3$.  

\begin{definition}\label{def: smoothed hinge loss}
The approximation to the hinge-loss is:
\[ \begin{array}{cc}
   L_\delta(t)= \left\{\begin{array}{ll} 
								0  & t\geq 1  \\
	              \frac{(1-t)^2}{4\delta} & 1-2\delta\leq t<1  \\ 
								1-t-\delta & t<1-2\delta
	 \end{array}\right. \\
\end{array} \]
\end{definition}

\begin{lemma}\label{lemma:smooth_hinge_loss}
The function in Definition \ref{def: smoothed hinge loss} is a once differentiable convex function.
\end{lemma}
 
\begin{proof}
The first derivative of the function in Definition \ref{def: smoothed hinge loss} is:
\[ \begin{array}{cc}
   L_\delta(t)= \left\{\begin{array}{ll} 
								0  & t\geq 1  \\
	              \frac{t-1}{2\delta} & 1-2\delta\leq t<1  \\ 
								-1 & t<1-2\delta
	 \end{array}\right. \\

\end{array} \]
The values of function $\frac{t-1}{2\delta}$ at $t=1,1-2\delta$ are $0,-1$ respectively. Thus the first derivative is continuous.

Convexity will be proven by showing that for every $t,t_0\in\mathbb{R}$, the inequality $L_\delta(t)\geq L_\delta(t_0)+\nabla L_\delta(t_0)(t-t_0)$ holds for $\delta>0$ (see \cite{boyd2004convex}).

For all pairs of $t$ and $t_0$ which are in the same region, this holds trivially because each of the functions is convex in itself. We will prove that for all the other combinations of $t$ and $t_0$  the inequality holds.

\begin{enumerate}
	\item $t\geq1$, $1-2\delta\leq t_0<1$: Need to prove- $0\geq \frac{(1-t_0)^2}{4\delta}+\frac{t_0-1}{2\delta}(t-t_0)=\left(\frac{t_0-1}{4\delta}\right)(2t-t_0-1)$. The first element is non-positive since $t_0<1$. The second element is non-negative since it is greater or equal to $2\delta$ due to $1-2\delta\leq t_0$.
	\item $1-2\delta\leq t<1$, $t_0\geq1$: Need to prove- $\frac{(1-t)^2}{4\delta}\geq 0$. This is always true for $\delta>0$.
	\item $t\geq1$, $t_0<1-2\delta$: Need to prove- $0\geq 1-t_0-\delta -1\times (t-t_0)=1-\delta-t\leq 1-\delta-1=-\delta$. Since $t\geq 1$ and $\delta>0$ this inequality is true.
	\item $t<1-2\delta$, $t_0\geq1$: Need to prove- $1-t-\delta\geq0$. $1-t-\delta\geq 1-\delta+2\delta-1=\delta\geq0$ since $t<1-2\delta$.
	\item $1-2\delta\leq t<1$, $t_0<1-2\delta$: Need to prove- $\frac{(1-t)^2}{4\delta}\geq 1-t_0-\delta-1(t-t_0)=1-t_0-\delta-t+t_0\geq 1-\delta+2\delta-1=\delta\geq 0$ since $t\geq1-2\delta$. 
	\item $t<1-2\delta$, $1-2\delta\leq t_0<1$:	Need to prove- $1-t-\delta\geq \frac{(1-t_0)^2}{4\delta}+\left(\frac{t_0-1}{2\delta} \right)(t-t_0)$. $1-t-\delta\leq\delta$ since $t<1-2\delta$. $\frac{(1-t_0)^2}{4\delta}+\left(\frac{t_0-1}{2\delta} \right)(t-t_0)=\left(\frac{t_0-1}{4\delta} \right) \left( 2t-t_0-1\right)$. $\frac{t_0-1}{4\delta}\leq 0$ since $t_0\leq1$. $2t-t_0-1\leq-2\delta\leq 0$ since $t<1-2\delta$ and $t_0\geq 1-2\delta$.
\end{enumerate}
\end{proof}

Note that the hinge-loss is a private case of the smoothed hinge loss with $\delta=0$. Thus the performance with this loss-function should be at least as good as with the hinge-loss with the possible additional hyper-parameter selection of $\delta$. Furthermore, with small $\delta$, this can be simply used as a differentiable hinge-loss with minimal change in performance.

Past work on differentiable approximations of the hinge loss placed the quadratic region in the positive domain (\cite{chapelle2007training}) and kept the negative domain the same as in the hinge loss. This provides an even larger emphasis on training examples within the margin which are not errors compared to training errors. In Section \ref{Subsection: Results} we demonstrate that utilizing a training error loss which provides less emphasis on examples in the margin which are not training errors can provide better performance compared to the trade-off used in the hinge loss.

\begin{figure}[h!]
\centering
\includegraphics[width = 0.7\textwidth]{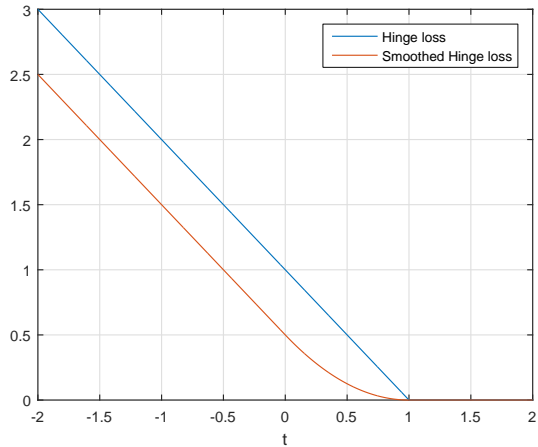}
\caption{Smoothed Hinge loss function, with $\delta=0.5$}
\label{fig:Smoothedhinge}
\end{figure}

Thus we propose to solve Problem \ref{opt_prob:primal_huber_smoothed_hinge} if the ultimate objective is per example and not per groups of examples.

\begin{equation}
\begin{array}{ll}\label{opt_prob:primal_huber_smoothed_hinge}
\mathrm{minimize} & \frac{1-\lambda}{d} \sum_{j=1,...d} h_\epsilon(w_j)+\frac{\lambda}{n_+}\sum_{i\in i_+}{L_\delta(y^i (w^T x^i+b))} +\frac{\lambda}{n_-}\sum_{i\in i_-}{L_\delta(y^i (w^T x^i+b))} \\
.                                                                              
\end{array}
\end{equation}

\begin{theorem} \label{theorem:NonGroupOpt}
Problem \ref{opt_prob:primal_huber_smoothed_hinge} is a convex optimization problem.
\end{theorem}

\begin{proof}[Sketch]
Problem \ref{opt_prob:primal_huber_smoothed_hinge} is an unconstrained optimization problem whose objective function is comprised of a non-negative sum of convex functions. The Huber function,$h_\epsilon$, is a known convex function of its argument for all $\epsilon>0$ (\cite{boyd2004convex}). The function $L_\delta(t)$ is a convex function as proved in Lemma \ref{lemma:smooth_hinge_loss}. Convexity is preserved under affine transformations.
\end{proof}

\subsection{Grouping examples}

The objective of the grouping of examples is to minimize the maximal loss error for each group. Any error which is not maximal, is unimportant.  Thus the new optimization problem is Problem \ref{opt_prob:primal_huber_smoothed_hinge_group}.

\begin{equation}
\begin{array}{ll}\label{opt_prob:primal_huber_smoothed_hinge_group}
\mathrm{minimize} & \frac{1-\lambda}{d} \sum_{j=1,...d} h_\epsilon(w_j)+ \\
 & \frac{\lambda}{n_+}\sum_{k\in k_+}{max_{i\in G_{k_+}}{\{L_\delta(y^i (w^T x^i+b))\}}} + \\
& \frac{\lambda}{n_-}\sum_{k\in k_-}{max_{i\in G_{k_-}}{\{L_\delta(y^i (w^T x^i+b))\}}} \\
,                                                                              
\end{array}
\end{equation}

where $G_{k_+}, G_{k_-}$ are groups of examples related to the positive and negative classes. Note that the number of elements in each group can be different and that $n_+$ and $n_-$ indicate the number of positive and negative \textit{groups} and not examples. 



Note that the maximum function turns the objective function of Problem \ref{opt_prob:primal_huber_smoothed_hinge_group} into a non-differentiable objective function. One path to handle this issue is to use a smoothed maximum function. However, if we were to use the sub-gradient instead, we could calculate the gradient only for the maximal soft-margin. This considerably reduces the computational load and we recommend using this option.

\subsection{Use case}

In this section we will demonstrate how to use the grouping on an image classification task. The object to detect is a cat. Therefore positive examples are images which contain at least one cat and negative examples are images which do not. The grouping extends to other classification task as well.

Consider the following algorithm for image classification:
\begin{enumerate}
	\item Region proposal algorithm which suggest candidates of an object in an image.
	\item An algorithm which provides a score indicative if the candidate is the desired object or not.
	\item A score per image which is the probability the image contains the object based on the maximal score of the candidates in the image.
\end{enumerate}

With non-end-to-end learning, the second step is divided into a features design stage and classifier stage and the score can be the soft-margin of the classifier. With end-to-end learning such as deep learning some or all of the steps can be combined. We will focus the discussion in this paper on algorithms that maintain at least the separation between the candidates level (steps $1$,$2$) and the image level (step $3$). This separation is typically required even with deep learning when the objects to detect comprise a small portion of the image.

A visualization of the scheme of the algorithm is presented in Figure \ref{fig:animals}. The maximum function is used to transition from the candidate level to the image level since it is enough for one candidate to contain the desired object in order to decide that the image contains the object.  

\begin{figure}[t!]
\centering
\includegraphics[width = \textwidth]{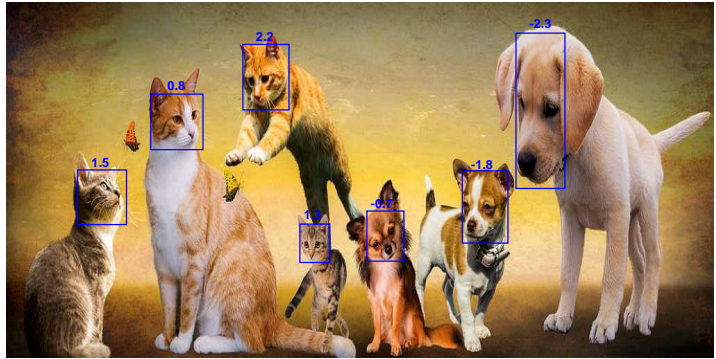}
\caption{The question is if there is (at least) one cat in the image. Consider a region proposal which provides rectangular candidates of animal faces. An algorithm provides a score per candidate. The overall image score is the maximum, which in this example is equal to $2.2$. If this is above a threshold, the image is termed to contain a cat.}
\label{fig:animals}
\end{figure}

Under this scheme, we propose that only the negative examples will be grouped per image, while for the positive examples only the candidate which best fits the desired object will be kept and never grouped with other examples. This suggestion is naturally only for the training phase and the test phase remains the same for any group of examples (e.g. image, video, etc.).

\subsubsection{Why group the negative examples?}

In the proposed image classification framework, all that is important is the maximal score per image. If the maximal score in a negative image will have a higher value than the correct candidate in a positive image, there will be a miss-classification. Figure \ref{fig:group_neg} presents an example to demonstrate this principle with the hinge-loss for simplicity. Consider that the positive image is \ref{fig:pos1N} and the negative image is \ref{fig:neg1N}. This means that one bad error in the negative example will cause an error on the image level, no matter how well the performance is for the rest of the candidates in the negative example. Therefore it is important to minimize the maximal hinge-loss for the negative examples.

Consider the situation where the negative image is \ref{fig:neg2N} instead of \ref{fig:neg1N}. While on \textit{average} the hinge-error loss is higher for \ref{fig:neg2N} compared to \ref{fig:neg1N}, the performance on the image level is better. This can explain why the standard criteria which focuses on minimizing the error on average is not optimal if the final algorithm's success criteria is on the group level.

\begin{figure}[t]
    \centering
    \begin{subfigure}[t]{0.31\textwidth}
        \includegraphics[width=\textwidth]{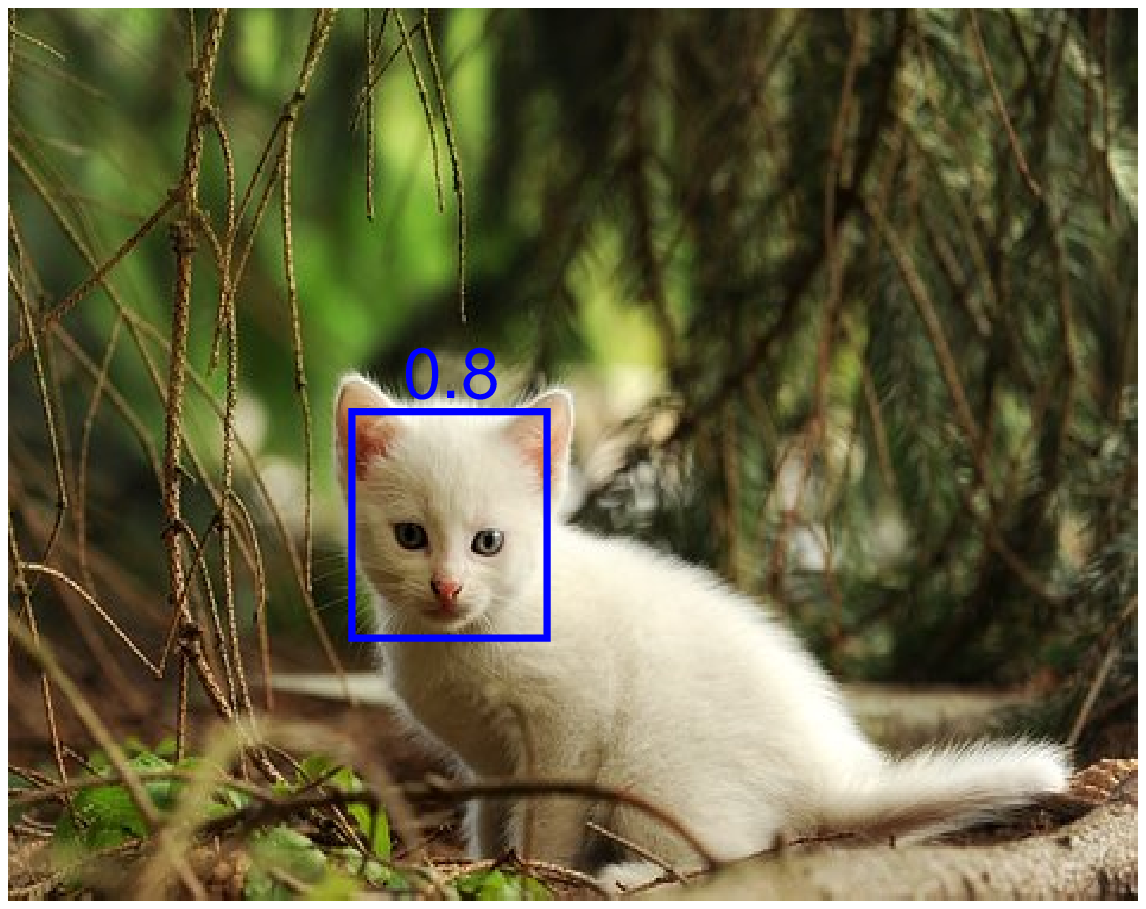}
        \caption{Positive example with 1 candidate with a hinge-loss of $0.2$.}
        \label{fig:pos1N}
    \end{subfigure}
    ~ 
    \begin{subfigure}[t]{0.31\textwidth}
        \includegraphics[width=\textwidth]{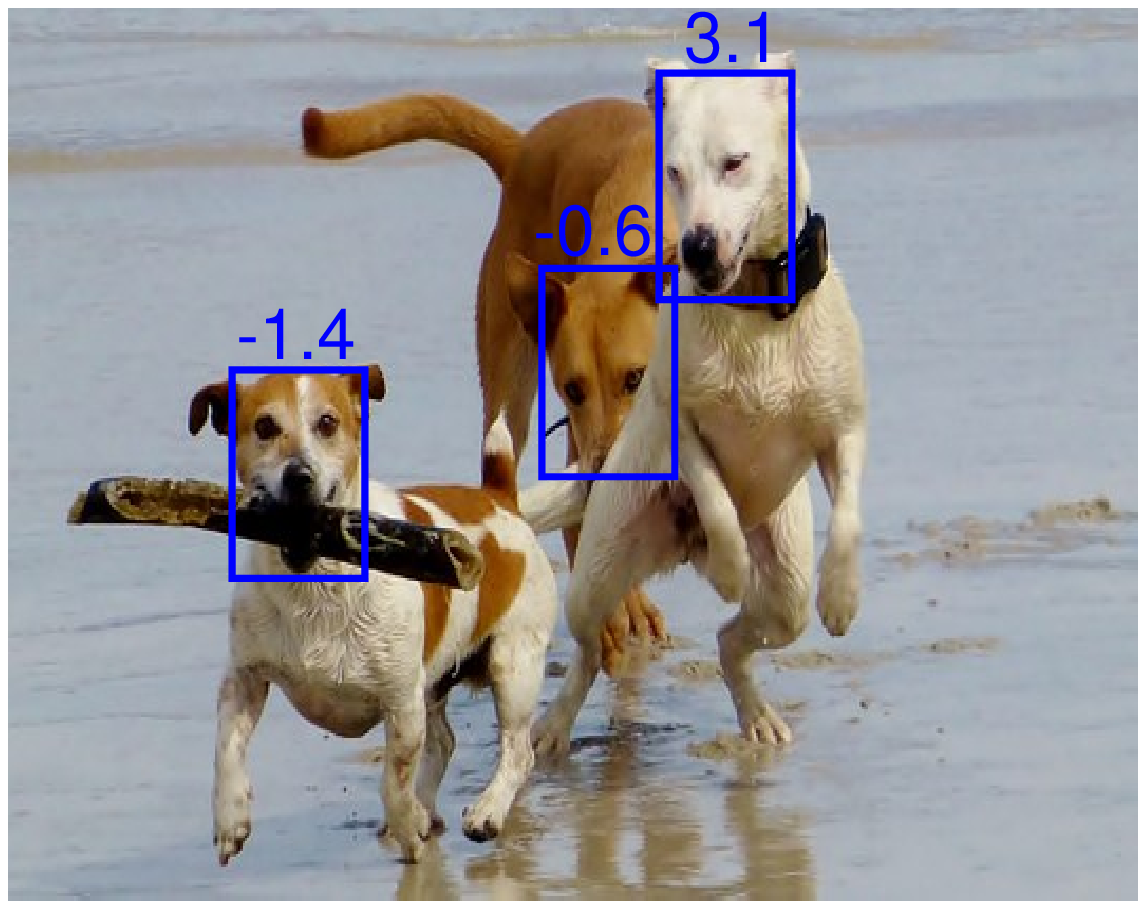}
        \caption{Negative example with 3 candidates with hinge-losses of $0$,$0.4$ and $4.1$ from left to right. The mean hinge-loss is $1.5$ and the max is $4.1$.}
        \label{fig:neg1N}
    \end{subfigure}
    ~ 
    \begin{subfigure}[t]{0.31\textwidth}
        \includegraphics[width=\textwidth]{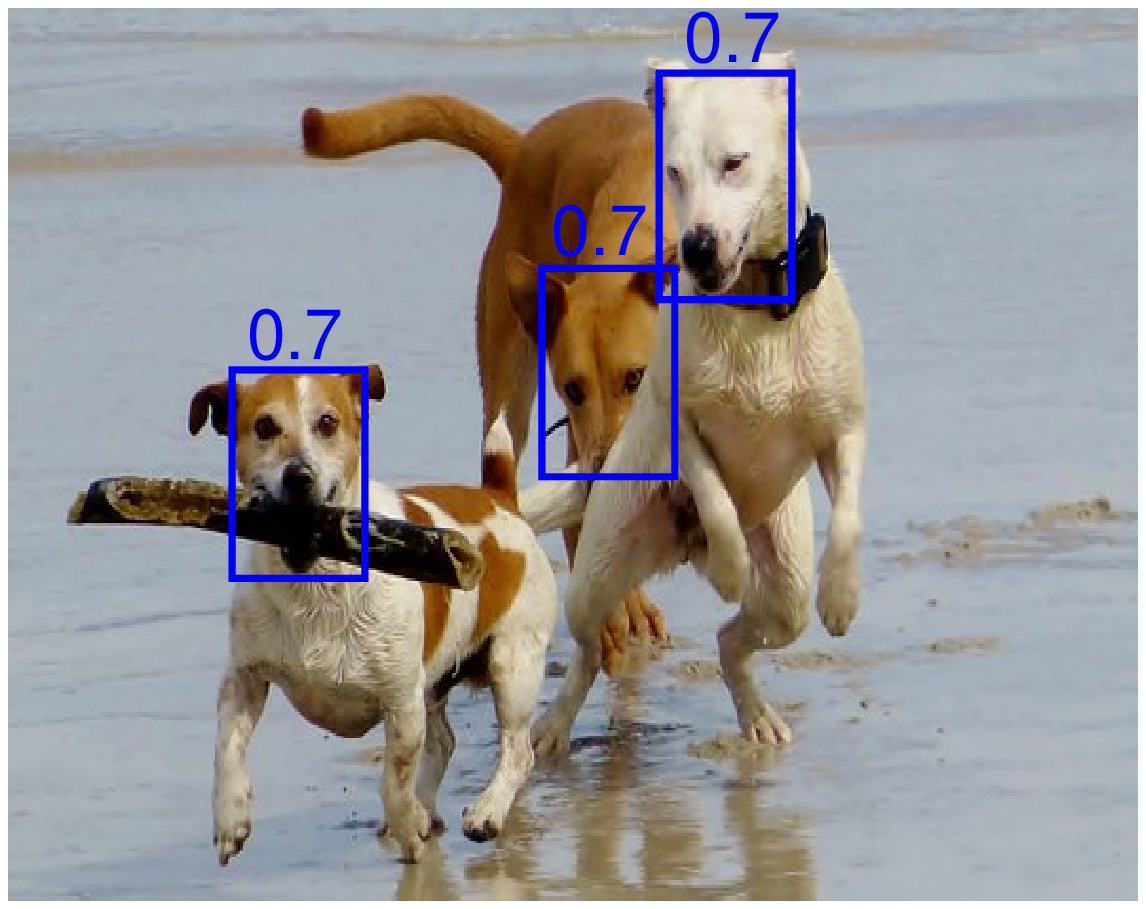}
        \caption{Same negative example as in (b) with hinge-losses $1.7$ for all. The mean and max hinge-loss is $1.7$.}
        \label{fig:neg2N}
    \end{subfigure}
    \caption{Grouping negative examples}\label{fig:group_neg}
\end{figure}

\subsubsection{Why not to group the positive examples?}

The true objective for the positive examples is to maximize the largest score for the positive image. Therefore minimizing the maximal penalty on the training error for a positive example, produces the exact opposite of the desired effect. For example, in Figure \ref{fig:group_pos} there are two positive images. The maximal hinge losses in images \ref{fig:pos1P} and \ref{fig:pos2P} are $1.4$ and $0.7$ respectively. This means that according to this criterion, option $b$ is preferable. However, the objective is that the maximal score of the candidates (preferably the candidate which best fits the object) will have a high score. Minimizing the maximal error for a positive example focuses on increasing the lowest score.

\begin{figure}[t!]
    \centering
    \begin{subfigure}[t]{0.47\textwidth}
        \includegraphics[width=\textwidth]{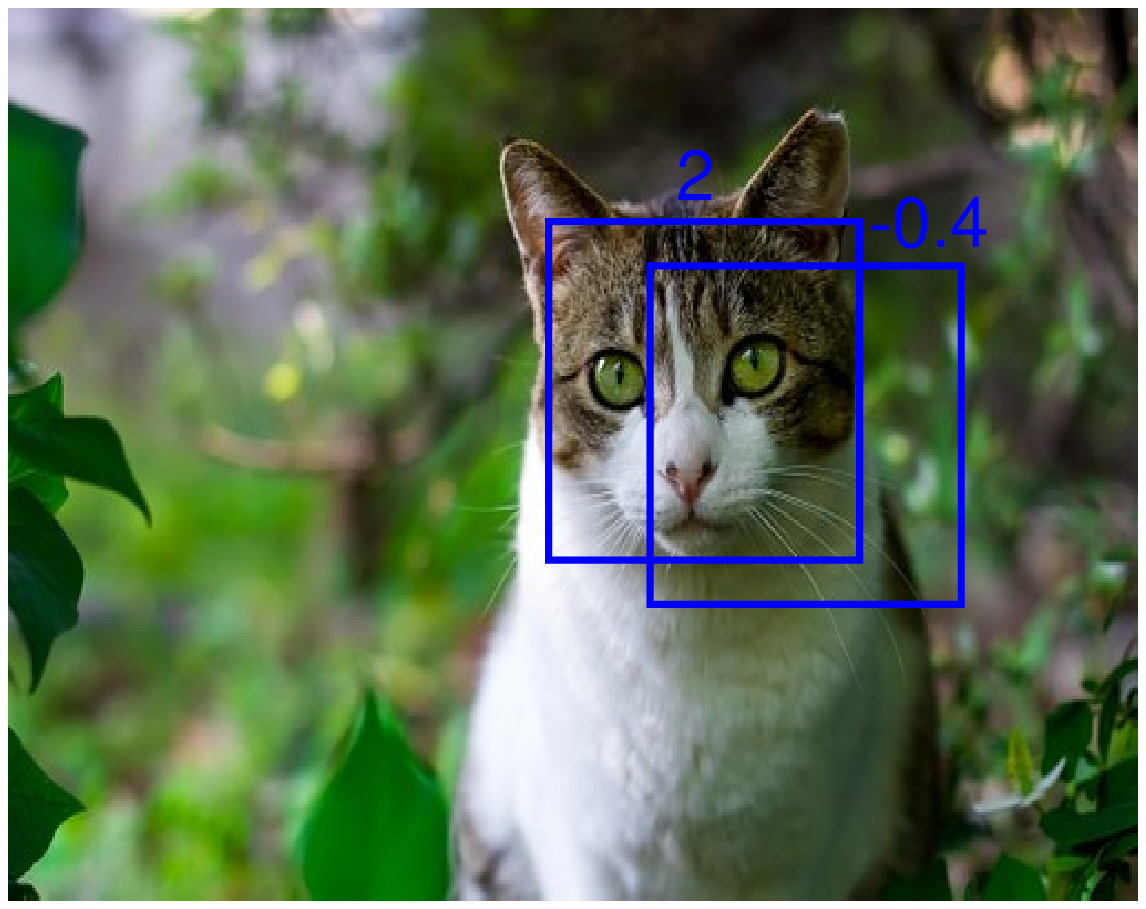}
        \caption{Positive example with 2 candidate with hinge-losses of $0$ and $1.4$.}
        \label{fig:pos1P}
    \end{subfigure}
    ~ 
    ~ 
    \begin{subfigure}[t]{0.47\textwidth}
        \includegraphics[width=\textwidth]{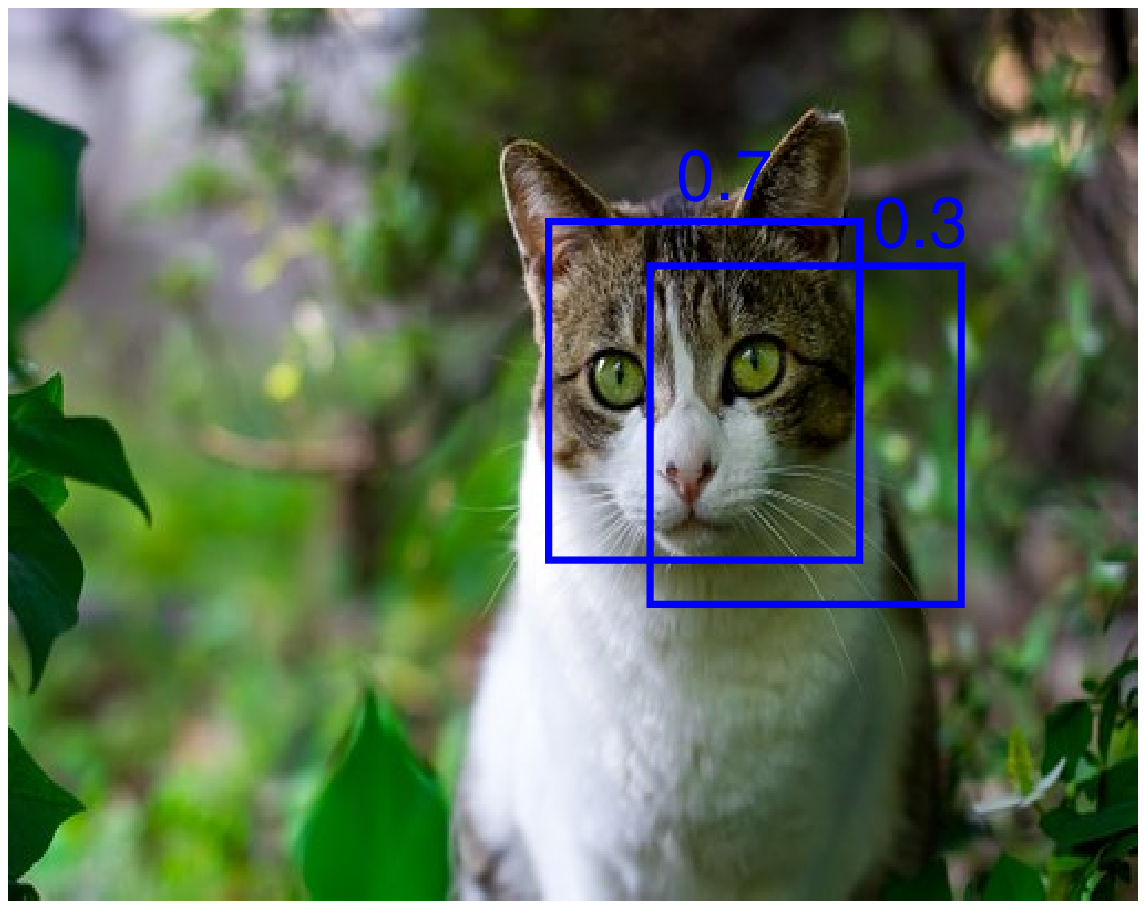}
        \caption{Same positive example as in (a) with hinge-losses of $0.3$ and $0.7$.}
        \label{fig:pos2P}
    \end{subfigure}
    \caption{Why not to group positive examples}\label{fig:group_pos}
\end{figure}


One can try to minimize the average penalty on the training error. This means that all candidates in the group will be treated equally and the optimization will try to increase all of them. While this is better than the previous choice, it still focuses on candidates whose values are not important. For example, in Figure \ref{fig:group_pos} the average hinge-losses in images \ref{fig:pos1P} and \ref{fig:pos2P} are $0.7$ and $0.5$ respectively. This means that according to this criterion, option $b$ is preferable. 

The best optimization criteria is minimizing the minimal penalty on the training error. For example, in Figure \ref{fig:group_pos} the minimal hinge-loss in images \ref{fig:pos1P} and \ref{fig:pos2P} are $0$ and $0.3$ respectively. This means that according to this criterion, option $a$ is preferable. This is the preferable option due to the min function on the group level. However, minimizing the minimum of a set of convex functions is a non-convex optimization problem.

While solving non-convex, even large scale, optimization problems has become common practice in deep learning, this approach has its disadvantages and it is best to refrain from using it there is no significant harm to the global minimum of the non-convex optimization problem. The disadvantages include:

\begin{enumerate}
	\item Only convergence to a local minima is guaranteed.
	\item In order to improve performance, often manual "baby-sitting" is required to monitor the progress of the algorithm.
	\item Performance depends on the initial guess.
\end{enumerate}

In the context of the image classification problem, this paper assumes that a mask is provided to identify the location of the desired object in the positive images. This paper proposes to minimize the penalty error of the candidate which best fits the provided mask. This is important for generalization, since maximizing the score of a candidate which is not on the desired object in the positive image may provide better performance on the image level in the training set. However, this will inevitably diminish performance on the test set. Thus we expect that there is little harm to the global minimum of the non-convex optimization problem.

Furthermore, if it is also important to provide an indication of the location of the reason why the image was deemed to contain the object, this step is also beneficial. Note however, that this will provide the location of just one object, even if the image contains more. This can be overcome easily if a constant threshold is used and an object is deemed to be positive if its score is higher than the aforementioned threshold. On the other hand, this means that the ultimate optimization problem is unrelated to a group and then one should either solve Problem \ref{opt_prob:primal_huber_smoothed_hinge} or other techniques which don't utilize grouping.

\subsubsection{Proposed algorithm}

For negative groups, group the examples and minimize the maximal training error penalty. For positive examples, mark masks of the positive objects and minimize for each group only the training error penalty with the candidate which best fits the manual mask according to an overlap criterion of choice. Thus the algorithm consists of solving Problem \ref{opt_prob:primal_huber_smoothed_hinge_group_use_case}.

\begin{equation}
\begin{array}{ll}\label{opt_prob:primal_huber_smoothed_hinge_group_use_case}
\mathrm{minimize} & \frac{1-\lambda}{d} \sum_{j=1,...d} h_\epsilon(w_j)+\\
 & \frac{\lambda}{n_+}\sum_{k\in k_+^*}\{L_\delta(y^k (w^T x^k+b))\} +\\
 & \frac{\lambda}{n_-}\sum_{k\in k_-}{max_{i\in G_{k_-}}{\{L_\delta(y^i (w^T x^i+b))\}}} \\
,                                                                              
\end{array}
\end{equation}

where $k_+^*$ is an index of groups and indicates the candidate with the maximal overlap to the marked mask per group. Note that for images without sufficient overlap to the marked mask, the images may be omitted from the training set. Also $n_+$ and $n_-$ indicate the number of positive and negative \textit{groups} and not examples. 

For the test phase, the classifier is used on all the candidates and the score per image is the maximal score. Solving Problem \ref{opt_prob:primal_huber_smoothed_hinge_group_use_case} is termed as the Group Classification Machine (GCM) algorithm. Appendix A details the gradient of Problem \ref{opt_prob:primal_huber_smoothed_hinge_group_use_case} for gradient based optimization algorithms (\cite{bertsekas1999nonlinear}).

\subsection{Non linear large scale algorithm}

Linear classifiers typically are insufficient to provide the best possible performance unless the features inputted to them span a domain which separates well between classes. One example of such an exception is in transfer learning where a deep learning system is trained on a large benchmark and the last classification layer is removed. Then the values of the last layer serve as the features of a linear classifier for a different problem with relatively a few examples. However, in most cases we would like the ability to use non-linear classifiers. Otherwise, the advantage gained by optimizing for groups may be lost due to this limitation.

In order to achieve non-linearity, via representer theorem a kernel classifier is typically used. In cases where the number of training examples is extremely large, this is problematic due to computational speed and computer memory. There is a significant amount of work on computational techniques aimed at solving large scale (in the sense of the examples and not the features) SVMs (\cite{achlioptas2002sampling}, \cite{tsang2005core}, \cite{blum2006random}, \cite{zanni2006parallel}, \cite{loosli2007training}, \cite{bottou2007large}, \cite{rahimi2008random}, \cite{wang2011trading}, \cite{vedaldi2012efficient}, \cite{zhang2012scaling}, \cite{li2012chebyshev}, \cite{yang2012nystrom}, \cite{Mu_2014_CVPR}, \cite{tu2016large}, \cite{lu2016large}, \cite{ionescu2017large}).

We propose to perform a non-linear transformation of features into a new set of features and maintain the linear classifier. Thus, the classifier will be non-linear in relation to the original classifier and the convexity of the optimization problem will remain. Although this step entails increasing the number of variables to optimize, the increase is related to the complexity of the classifier needed and not to the number of training examples. 

One example is to transform the input features to all the polynomial factors up to a certain degree. This results in a polynomial classifier in relation to the original features. In order to exemplify, consider two features $x_1,x_2$ with a maximal degree of $3$. This results in 9 features $\{x_1,x_2,x_1^2,x_2^2,x_1^3,x_2^3,x_1 x_2,x_1 x_2^2,x_1^2 x_2\}$. This polynomial transform is an exponential increase in the number of features, however, if the number of features is too large to optimize, one can prune the number of combinations or lower the maximal degree.
The advantage of this approach is that any non-linear transformations can be performed and the classification problem will remain convex.

\subsection{Comparison to MI-SVM}

The MI-SVM algorithm \cite{andrews2003support} has the following form:
\begin{equation}
\begin{array}{ll}\label{opt_prob:MI_SVM_original}
\mathrm{minimize} & w^Tw+C (\sum_{i\in i_+}\xi_{i} + \sum_{i\in i_-}\xi_{i}) \\
\mbox{subject to} & \xi_{i} \geq \mathbf{1} + (X^{i} w + b\mathbf{1}) ,\hspace{0.2cm} \forall{i\in i_-} \\
\mbox{subject to} & \xi_{i} \geq \mathbf{1} - (x^{s_{i}} w + b\mathbf{1}), \hspace{0.2cm} \forall{i\in i_+} \\
                  & \xi \geq 0,                                                                                   
\end{array}
\end{equation}

where the variables are $w\in \mathbb{R}^{d}$, $b\in \mathbb{R}$, $\xi\in \mathbb{R}^{n}$, and the selector variable $s\in \mathbb{Z}^{n_+}$ which chooses which candidate has the highest soft-margin for each group. $X^{i}$ is the matrix of all the feature vectors in a group and $x^{s_{i}}$ is the feature vector of the group with the largest softmargin.

This is a mixed integer optimization problem \cite{boyd2004convex} and in \cite{andrews2003support} it was proposed to iteratively guess which feature vector in each group will represent the group and optimize Problem \ref{opt_prob:MI_SVM_original} after setting $s$. The first guess is the average of the feature vectors for each positive group. The rest of the steps use the previous step's soft-margins to determine which is the maximal in each group. The stopping criteria is when the selector variable $s$ doesn't change.

The main differences between MI-SVM and GCM are:
\begin{enumerate}
\item In GCM there is no need for the selector variable. The candidate which should receive the highest score in a positive group is known thanks to the annotation. While this has some annotation cost, this can lead to a significant improvement in performance especially when there aren't many positive groups and\textbackslash or if the positive groups consist of only a small percent of positive candidates.
\item GCM is an unconstrained optimization problem which can be optimized for very large scale problems. This is critical when each group consists of many candidates. Without the ability to handle a large scale of candidates, the number of group for training would be very limited.
\item GCM solves the global optimization problem (it focuses on the candidate the supervisor wants to represent the positiveness of the group) of the MI-SVM with one convex optimization problem as opposed to the MI-SVM which requires solving a sequence of convex optimization problems without guarantee of the global solution.
\end{enumerate}

Thus GCM can't be considered as a special case of the MI-SVM algorithm since it requires stronger supervision than MI-SVM and has the potential to improve the performance with lower computational cost.

\section{Experiments}\label{Section:Experiments}

The performance of the proposed GCM algorithm in relation to its origin, the SVM classifier, and the MI-SVM will be demonstrated on the image classification task of polyp detection in capsule endoscopy data. The results are provided for the linear classifier in order to fairly compare between the algorithms, since the non-linear versions of the SVM, MI-SVM and GCM algorithms are not identical. 

Any gradient descent algorithm can be used to optimize Problem \ref{opt_prob:primal_huber_smoothed_hinge_group_use_case}. In this paper, the quasi-newton limited FBGS with Armijo line-search was used \cite{bertsekas1999nonlinear}.

For both algorithms cross validation was used to tune the hyper-parameter which trades-off between generalization and training errors ($C$ for SVM, MI-SVM and $\lambda$ for GCM). The parameter of the Huber function $\epsilon$ was set to $1$ and the parameter of the soothed hinge-loss $\delta$ was set to $0.5$. From previous experiments we observed that the performance does not change significantly for different values of $\epsilon$ and $\delta$, so in order not increase significantly the required cross validation, these values were fixed beforehand.

\subsection{The challenge of polyp detection in capsule endoscopy}

Capsule Endoscopy for the colon is an emerging market with the potential to become a screening tool for detecting colorectal cancer. The patient swallows a pill that includes cameras which provide a video of the digestive system. If a significant polyp is detected with the capsule, the patient will be referred to Colonoscopy to have it treated.
Polyps are typically focal pathologies which can appear in only a few frames on the screen. Thus they can be missed by the physician. Currently physicians are provided videos and decide if to refer to Colonoscopy for treatment or not. We utilized polyp detectors which helped us to omit frames while maintaining the diagnostic yield. Figure \ref{fig:given_imaging} demonstrates the capsule and how it is used.

\begin{figure}[t!]
    \centering
    \begin{subfigure}[t]{0.47\textwidth}
        \includegraphics[width=\textwidth]{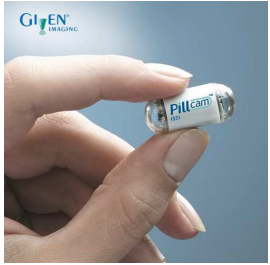}
        \caption{A person swallows a colon PillCam.}
        \label{fig:given1}
    \end{subfigure}
    ~ 
    ~ 
    \begin{subfigure}[t]{0.47\textwidth}
        \includegraphics[width=\textwidth]{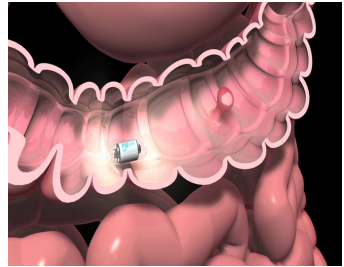}
        \caption{The capsule on the lookout for polyps.}
        \label{fig:given2}
    \end{subfigure}
    \caption{The capsule travels naturally through the digestive system, transmitting images of the GI tract.}\label{fig:given_imaging}
\end{figure}

Polyps may be missed, since many healthy parts of the gastrointestinal tract are similar to them. Foremost are the natural folds in the colon, which can be very similar to polyps in appearance and which are present in almost every frame. Figure \ref{fig:polyps} demonstrates how polyps look like and how difficult it is differentiate between them and healthy tissue. The decision that needs to be made is if to display the acquired frame to the user or not. Therefore this is an image classification problem.

\begin{figure}[t!]
    \centering
    \begin{subfigure}[t]{0.3\textwidth}
        \includegraphics[width=\textwidth]{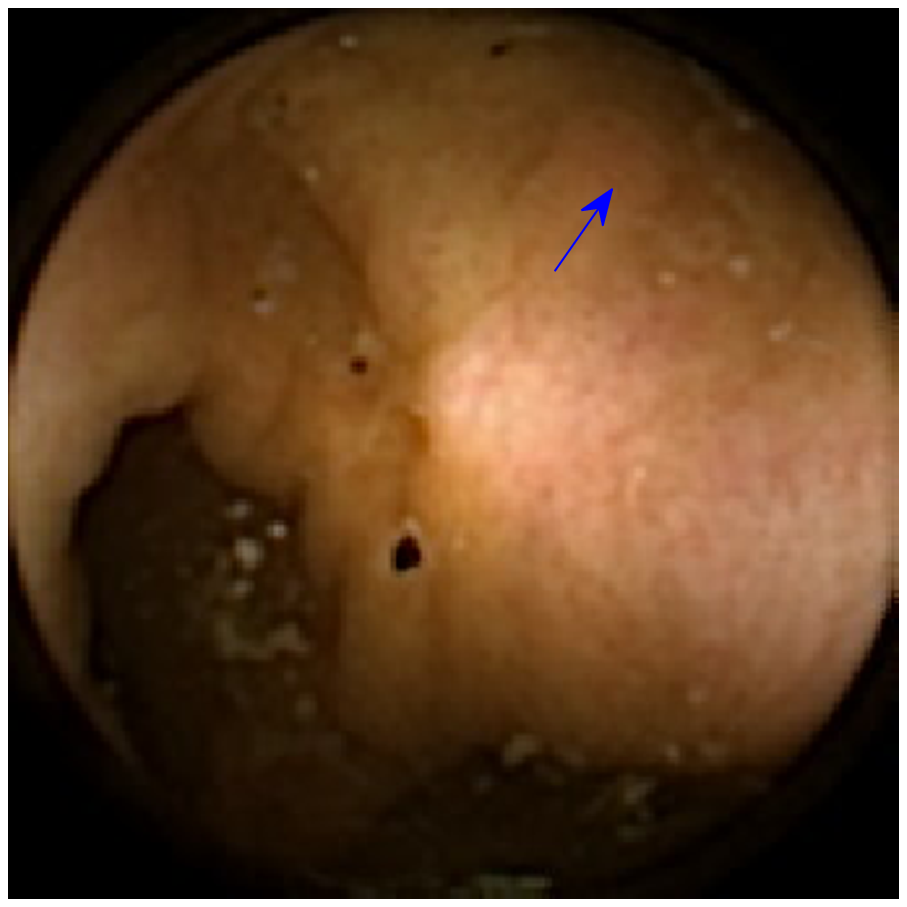}
        \label{fig:polyp1}
    \end{subfigure}
    ~ 
    \begin{subfigure}[t]{0.3\textwidth}
        \includegraphics[width=\textwidth]{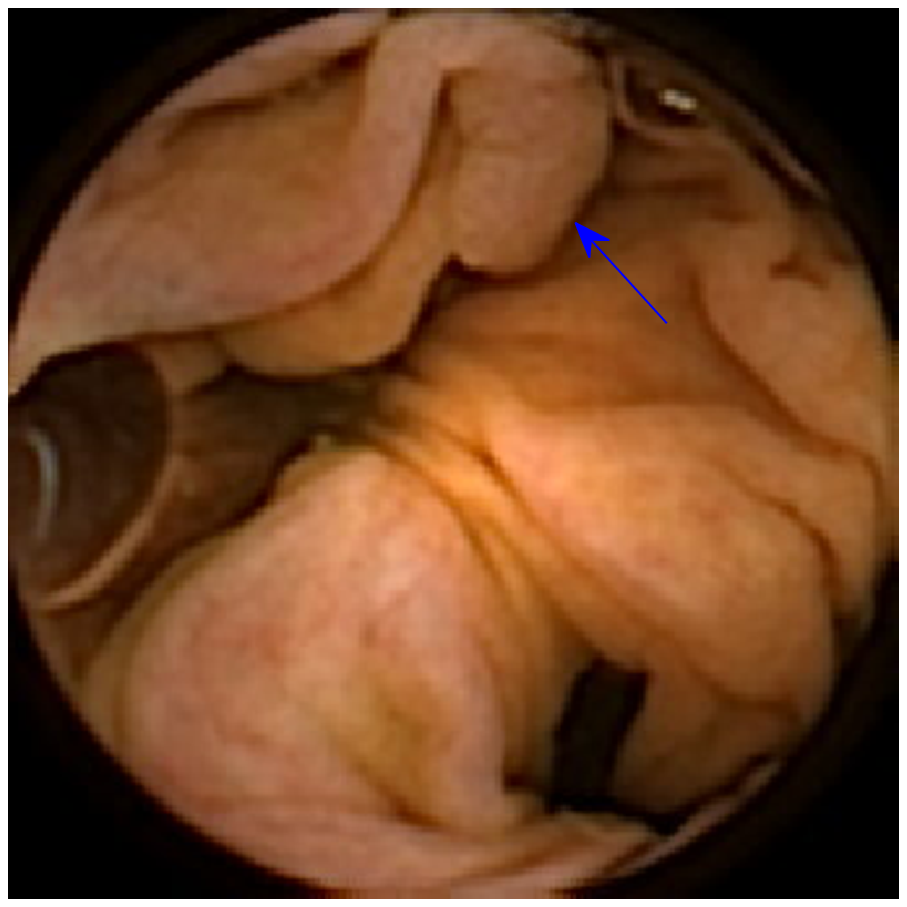}
        \label{fig:polyp2}
    \end{subfigure}
    ~ 
    \begin{subfigure}[t]{0.3\textwidth}
        \includegraphics[width=\textwidth]{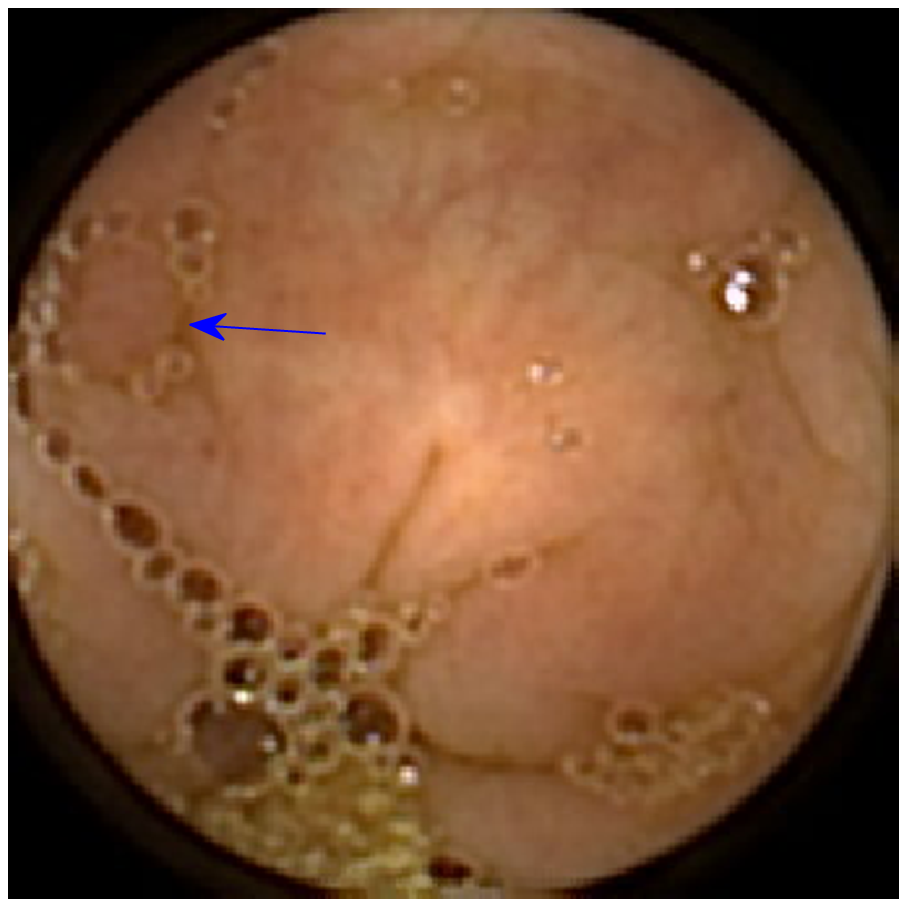}
        \label{fig:polyp3}
    \end{subfigure}
    \caption{Examples of images of the colon with arrows indicating polyps.}\label{fig:polyps}
\end{figure}

In this work there was a segmentation phase which resulted in $\sim 200$ candidates per image on average. For each candidate $13$ high level features were designed to best differentiate between the positive candidates and the negative candidates with an SVM classifier.

\subsection{Results}\label{Subsection: Results}

Table \ref{table:datasets} details the training and test sets used to compare the SVM, MI-SVM and GCM algorithms. While the test set is identical for all algorithms, the negative training set consists of $20000$ times more data for the GCM algorithm compared to the SVM algorithm. The reason the training set size couldn't be large for the standard SVM algorithm was due to memory constraints caused by the need to store the kernel matrix. For the GCM algorithm, there was no such limitation. Note that a test with the same limited negative training set for the GCM algorithm is described later on in Figure \ref{fig:per_cand_vs_per_group_training_size}.

\begin{table}
\centering
\begin{tabular}{ |c|c|c|c|c|c| } 
\hline
Classifier & & SVM & MI-SVM & GCM \\
\hline
\multirow{2}{4em}{Training} & Positive & 293 (100 polyps) & 5000 (25 polyps)  & 293 (100 polyps) \\ 
                            & Negative & 5000             & 5000              & 1e8 \\ 
\hline
\multirow{2}{4em}{Test}     & Positive & 362 (115 polyps) & 362 (115 polyps) & 362 (115 polyps) \\ 
                            & Negative & 326e6            & 326e6            & 326e6 \\ 
\hline
\end{tabular}
\caption{Training and test sets.}
\label{table:datasets}
\end{table}

The MI-SVM algorithm had the same negative examples for training as did the SVM algorithm. However it had many more examples in order to work on the group level for the positive examples. Thus it was provided with a quarter of all the available positive groups and for each group it received all of the candidates. Due to the structure of the optimization problem, no more positive training examples could be provided.

Figure \ref{fig:per_cand} details the performance with optimization on the candidate level. This means that all the positive and negative examples were treated equally without regard to the source of their images. Both versions of the GCM algorithm outperform the SVM and MI-SVM algorithm. Note that the GCM without grouping achieves the better performance. This is because the success criteria for this graph is the candidate level and not the image level. Note that the MI-SVM algorithm was not designed to optimize performance on the candidate level, so this can explain why it's performance is lower than the SVM algorithm.

Figure \ref{fig:per_group} details the performance with optimization on the image level. Again both versions of the GCM are better than the SVM and MI-SVM, but this time the GCM with grouping outperforms the version without grouping. This is because the success criteria in this figure is the image level. The MI-SVM algorithm has inferior performance even compared to the SVM algorithm despite the fact that it is optimized for the group level. This can be explained due to the low number of examples on the group level and the ambiguity of having to choose which candidate are positive for each positive group. For the SVM and GCM algorithm the positive candidates were identified by design. The combination of a low number of groups and less than $1\%$ positive candidates per positive group may be the cause of the difference between the MI-SVM to the rest of the algorithms.

\begin{figure}[t!]
    \centering
    \begin{subfigure}[t]{0.47\textwidth}
        \includegraphics[width=\textwidth]{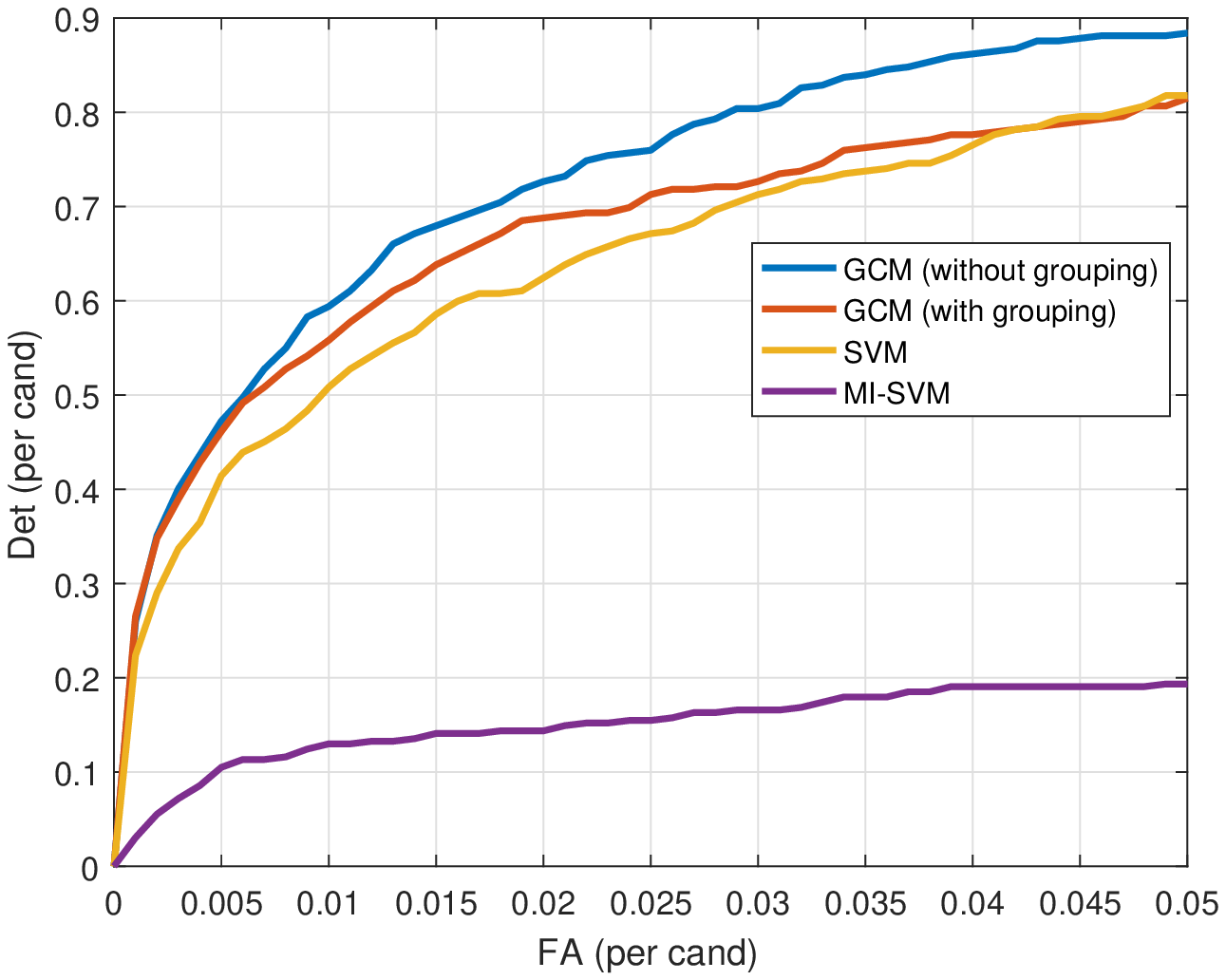}
        \caption{Performance per candidate.}
        \label{fig:per_cand}
    \end{subfigure}
    ~ 
    ~ 
    \begin{subfigure}[t]{0.47\textwidth}
        \includegraphics[width=\textwidth]{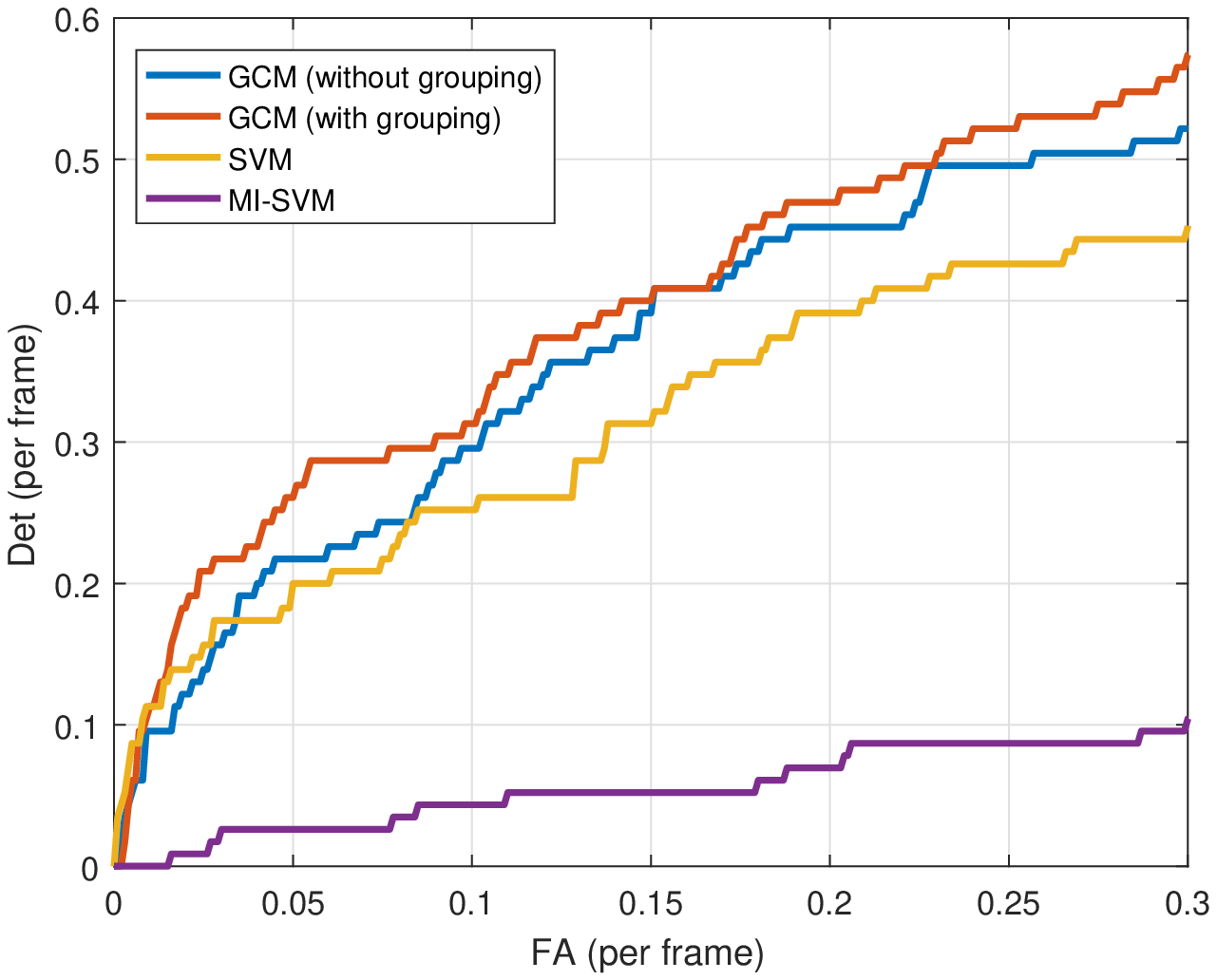}
        \caption{Performance per group.}
        \label{fig:per_group}
    \end{subfigure}
    \caption{Comparison of the performance in the candidate and group level}\label{fig:per_cand_vs_per_group}
\end{figure}

The conclusions from Figure \ref{fig:per_cand_vs_per_group} are:
\begin{enumerate}
	\item GCM may have better performance than SVM and MI-SVM, regardless of the grouping option.
	\item One should choose the grouping option of the GCM according to the success criteria.
\end{enumerate}

Another important aspect of the algorithm is the dependency of its performance relative to the training set size. Also one may think based on Figure \ref{fig:per_cand_vs_per_group} that the improved performance compared to the SVM algorithm is only due to the significantly larger number of negative training examples. Figure \ref{fig:per_cand_training_size} compares the performance of the GCM without grouping on the candidate level. It is clear that the GCM is better than the SVM even with the same limited training set which consists of only $5000$ negative training examples. This means that the objective function with the different penalty terms on the generalization (Huber function) and training errors (smoothed hinge-loss function) leads to better performance in itself on this dataset. However, there is no noticeable increase in the performance for the GCM without grouping on the candidate level by increasing the negative examples from $5000$ to $1e8$.

\begin{figure}[t!]
    \centering
    \begin{subfigure}[t]{0.47\textwidth}
        \includegraphics[width=\textwidth]{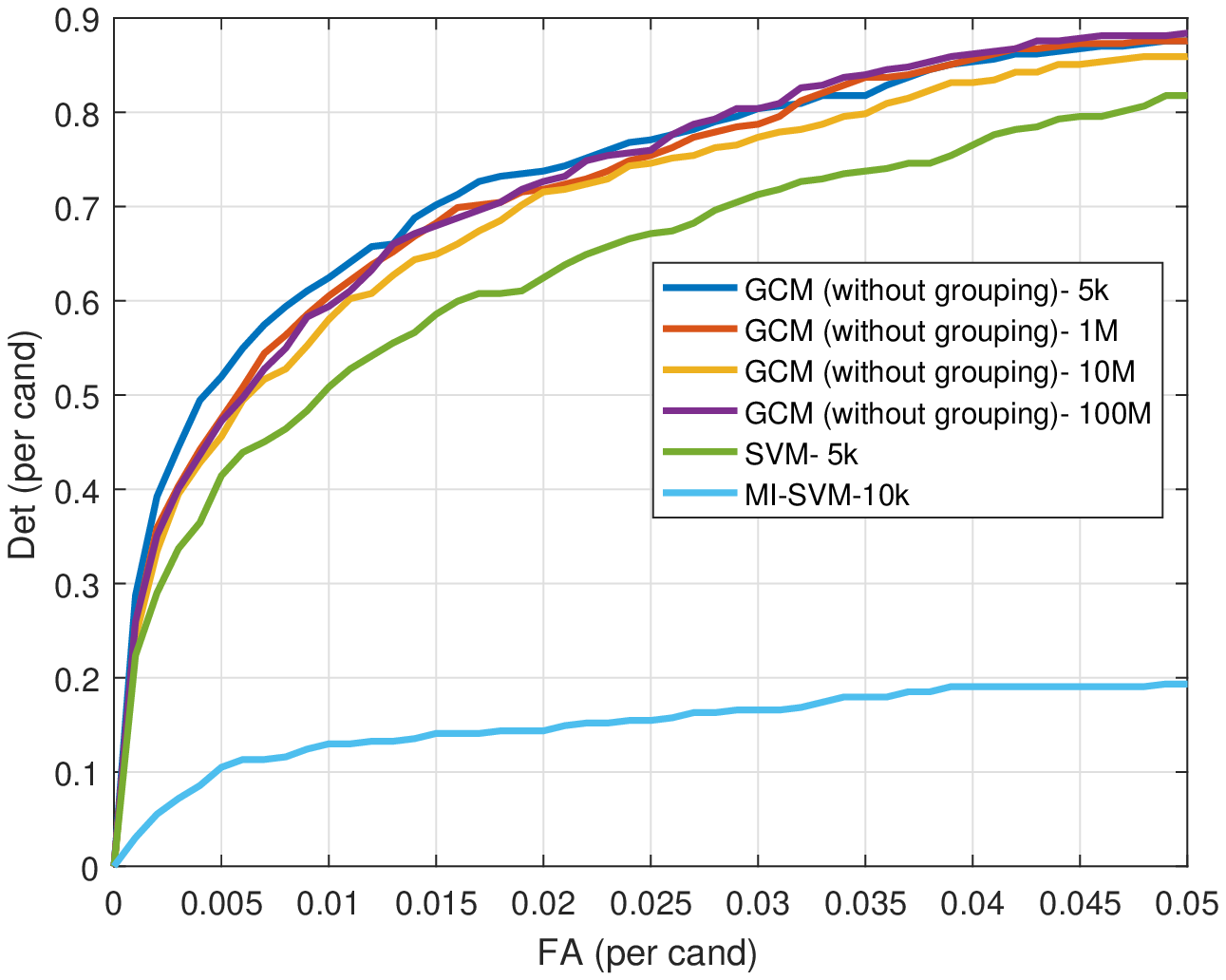}
        \caption{Performance per candidate dependency on training set size.}
        \label{fig:per_cand_training_size}
    \end{subfigure}
    ~ 
    ~ 
    \begin{subfigure}[t]{0.47\textwidth}
        \includegraphics[width=\textwidth]{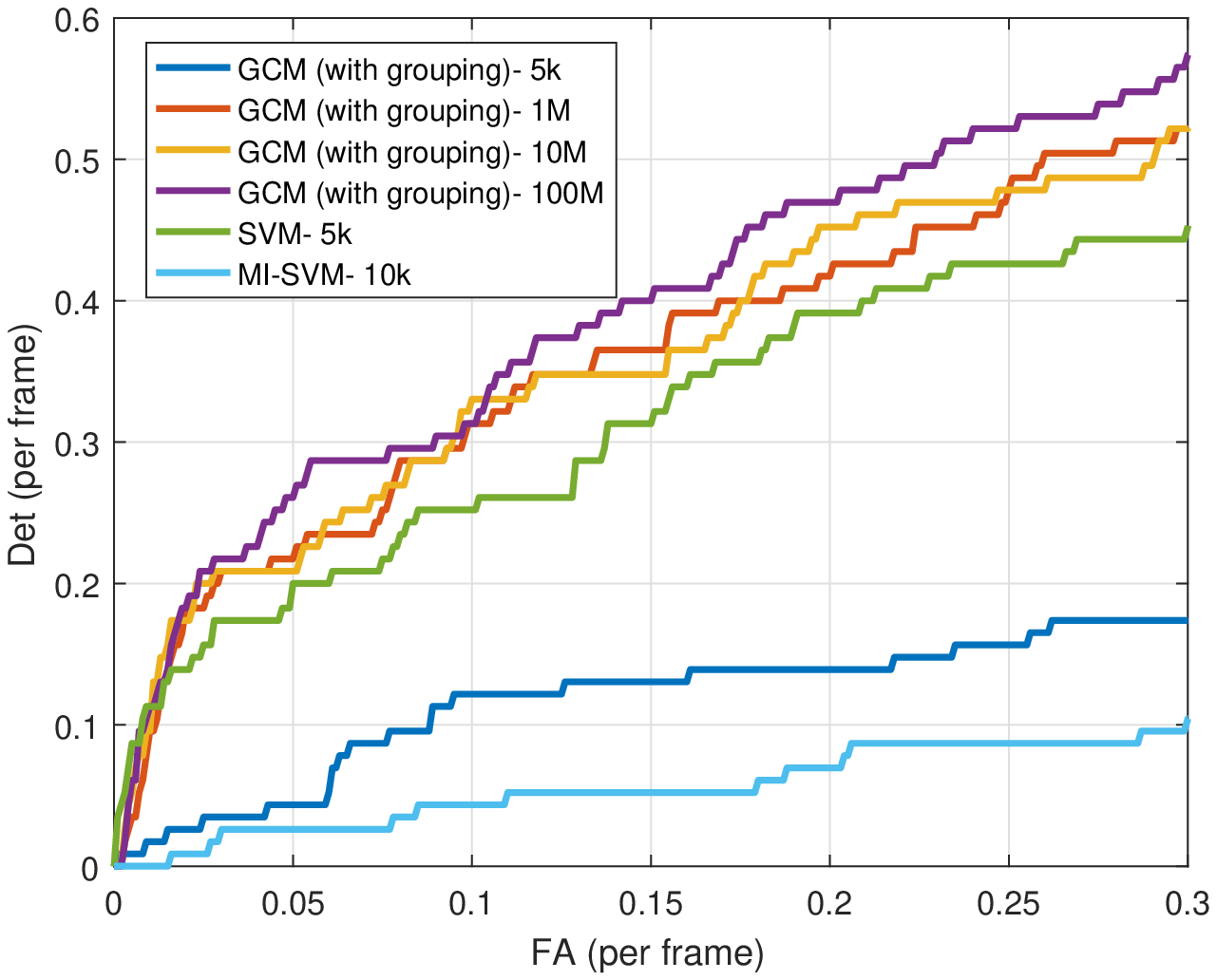}
        \caption{Performance per group dependency on training set size.}
        \label{fig:per_group_training_size}
    \end{subfigure}
    \caption{Comparison of performance in the candidate and group level as a function of the training set size}\label{fig:per_cand_vs_per_group_training_size}
\end{figure}

On the other hand, from Figure \ref{fig:per_group_training_size} it is clear that there is a benefit from increasing the number of negative training examples dramatically. Note that with only $5000$ examples, the grouping GCM is worse than the SVM classifier. This is due to the fact that it utilizes only $25$ images since there are roughly $200$ candidates per image.
Another interesting observation from Figure \ref{fig:per_group_training_size} is that there is still a significant gap between the GCM and the MI-SVM algorithm even when the training set for the GCM is limited as it is for the MI-SVM. This gap can be explained by the certainty of the GCM of which is candidate should have the maximal soft-margin in the positive groups.

\subsection{Why aren't there more experiments and comparisons to other MIL algorithms?}\label{Subsection: more experiments}

The framework of this paper is \textit{different} from the MIL framework where only the group labels are available. In this paper, with very little annotation effort, the algorithm has access to a reliable estimate of which candidate within a positive group in the training set makes the group positive.
This information is utilized by the GCM algorithm and for all the MIL algorithms, this information is part of the problem which they try to solve. Thus inherently the performance of the GCM algorithm will be better than that of MIL algorithms. 
This is equivalent to reducing the positive groups in MIL to just the single positive example.

One popular MIL algorithm, MI-SVM, was compared to GCM and indeed the performance of the latter was superior. It is obvious that this will be the same for the rest of the MIL algorithms.

There are no standard datasets where the positive candidate in a positive group is known. Once such a dataset is available, it easy to compare to MIL algorithms. However, the reverse is not possible. If only the group information is known, then the GCM can't be applied.

One alternative proposed in \cite{vanwinckelen2016instance} was to take standard datasets which don't have group information and have many more negative examples than positive and synthetically group them together. While this is an clever idea, this would mean in our case that the GCM would be provided exactly the right positive candidate in the positive groups as opposed to estimated one based on annotation as we presented. This would make the gap even larger compared to MIL algorithms and make the comparison even more unfair to MIL algorithms.

\section{Conclusions and future work}\label{Section:Summary}

In this paper a novel grouping technique of examples was presented which preserves convexity of the optimization problem. The benefit of utilizing the grouping option was demonstrated on an image classification problem of detecting polyps in capsule endoscopy images. 

While a non-linear extension has been introduced, even the linear classifier can be used for transfer learning with a similar dataset but when the desired data doesn't have enough (positive) examples to fine-tune the last few layers.
In future work we will design non-linear transforms on the data in order to perform end-to-end learning directly.

We recommend optimizing the non-grouping option when there is no meaning for groups in the classification task since it may outperform the classical SVM primal objective function and it is a smoother objective function.

The GCM demonstrated that while it requires more supervision than in a MIL setting, this typically requires fast and inexpensive annotation that can lead to significantly better performance with lower computational cost and global convergence to the positive candidates determined by the human supervisor. This may be especially true in cases where there is a relatively small number of positive groups (e.g. most problems in the medical domain) or when the percentage of positive candidates in positive groups is very small (e.g. less than $1\%$).






\appendix
\section*{Appendix A.}
\label{app:theorem}



In this appendix we detail the gradient of the objective function of Problem \ref{opt_prob:primal_huber_smoothed_hinge} which is the non-grouping option. Then we explain how to shift to a subgradient for Problem \ref{opt_prob:primal_huber_smoothed_hinge_group_use_case} which is the proposed optimization problem to solve for groups.

The gradient of $w$ is:

\begin{equation}
\begin{array}{ll}\label{gradient_w}
\nabla_w f_0(w,b)= & \frac{(1-\lambda)}{d} \nabla h_\epsilon(w)-\\
                   & \frac{\lambda}{n_+}\sum_{i\in i_+ \cap \in i_A}{x^i} +\\
                   & \frac{\lambda}{n_-}\sum_{i\in i_- \cap \in i_A}{x^i}+\\
                   & \frac{\lambda}{2\delta n_+}\sum_{i\in i_+ \cap \in i_M}(w^T x^i+b-1){x^i}+\\
                   & \frac{\lambda}{2\delta n_-}\sum_{i\in i_- \cap \in i_M}(w^T x^i+b+1){x^i}\\																		
,                                                                              
\end{array}
\end{equation}

where $i_A$ is the index of examples $x^i$ for which $y^i (w^T x^i+b)<1-2\delta$ and $i_M$ is the index of examples $x^i$ for which $1-2\delta\leq y^i (w^T x^i+b)<1$.

The gradient of $b$ is:

\begin{equation}
\begin{array}{ll}\label{gradient_b}
\nabla_b f_0(w,b)= & \frac{-\lambda}{n_+}\sum_{i\in i_+ \cap \in i_A} 1+ \\
                   & \frac{\lambda}{n_-}\sum_{i\in i_- \cap \in i_A} 1+\\
                   & \frac{\lambda}{2\delta n_+}\sum_{i\in i_+ \cap \in i_M}(w^T x^i+b-1)+\\
                   & \frac{\lambda}{2\delta n_-}\sum_{i\in i_- \cap \in i_M}(w^T x^i+b+1)\\                                                                             
\end{array}
\end{equation}

The sub-gradients of the objective function of Problem \ref{opt_prob:primal_huber_smoothed_hinge_group_use_case} are the same as of Problem \ref{opt_prob:primal_huber_smoothed_hinge}, except that the active set is the subset of $i_A$ whose argument is the maximal per group. Namely, $i_A$ is the index of examples $x^i$ for which $y^i (w^T x^i+b)<1$ and $y^i (w^T x^i+b)$ are the maximal values of per group of indices. It is beneficial to work with sub-gradients as opposed to using a smoothed approximation of the max function since this entails much smaller summation (proportional to the number of training examples divided by the number of groups).

\clearpage
\bibliography{mybibfile}{}

\begin{thebibliography}{10}
\expandafter\ifx\csname url\endcsname\relax
  \def\url#1{\texttt{#1}}\fi
\expandafter\ifx\csname urlprefix\endcsname\relax\def\urlprefix{URL }\fi
\expandafter\ifx\csname href\endcsname\relax
  \def\href#1#2{#2} \def\path#1{#1}\fi

\bibitem{dietterich1997solving}
T.~G. Dietterich, R.~H. Lathrop, T.~Lozano-P{\'e}rez, Solving the multiple
  instance problem with axis-parallel rectangles, Artificial intelligence
  89~(1) (1997) 31--71.

\bibitem{maron1998framework}
O.~Maron, T.~Lozano-P{\'e}rez, A framework for multiple-instance learning, in:
  Advances in neural information processing systems, 1998, pp. 570--576.

\bibitem{wang2000solving}
J.~Wang, J.-D. Zucker, Solving multiple-instance problem: A lazy learning
  approach.

\bibitem{gartner2002multi}
T.~G{\"a}rtner, P.~A. Flach, A.~Kowalczyk, A.~J. Smola, Multi-instance kernels,
  in: ICML, Vol.~2, 2002, pp. 179--186.

\bibitem{zhang2002dd}
Q.~Zhang, S.~A. Goldman, Em-dd: An improved multiple-instance learning
  technique, in: Advances in neural information processing systems, 2002, pp.
  1073--1080.

\bibitem{andrews2003support}
S.~Andrews, I.~Tsochantaridis, T.~Hofmann, Support vector machines for
  multiple-instance learning, in: Advances in neural information processing
  systems, 2003, pp. 577--584.

\bibitem{weidmann2003two}
N.~Weidmann, E.~Frank, B.~Pfahringer, A two-level learning method for
  generalized multi-instance problems, in: European Conference on Machine
  Learning, Springer, 2003, pp. 468--479.

\bibitem{xu2003statistical}
X.~Xu, Statistical learning in multiple instance problems, Ph.D. thesis, The
  University of Waikato (2003).

\bibitem{frank2003applying}
E.~Frank, X.~Xu, Applying propositional learning algorithms to multi-instance
  data.

\bibitem{auer2004boosting}
P.~Auer, R.~Ortner, A boosting approach to multiple instance learning, Lecture
  Notes in Computer Science 3201 (2004) 63--74.

\bibitem{zhou2004multi}
Z.-H. Zhou, Multi-instance learning: A survey, Department of Computer Science
  \& Technology, Nanjing University, Tech. Rep.

\bibitem{chen2006miles}
Y.~Chen, J.~Bi, J.~Z. Wang, Miles: Multiple-instance learning via embedded
  instance selection, IEEE Transactions on Pattern Analysis and Machine
  Intelligence 28~(12) (2006) 1931--1947.

\bibitem{nowak2006sampling}
E.~Nowak, F.~Jurie, B.~Triggs, Sampling strategies for bag-of-features image
  classification, Computer Vision--ECCV 2006 (2006) 490--503.

\bibitem{kriegel2006approach}
H.-P. Kriegel, A.~Pryakhin, M.~Schubert, An em-approach for clustering
  multi-instance objects., in: PAKDD, Vol.~6, Springer, 2006, pp. 139--148.

\bibitem{zhou2007solving}
Z.-H. Zhou, M.-L. Zhang, Solving multi-instance problems with classifier
  ensemble based on constructive clustering, Knowledge and Information Systems
  11~(2) (2007) 155--170.

\bibitem{kwok2007marginalized}
J.~T. Kwok, P.-M. Cheung, Marginalized multi-instance kernels., in: IJCAI,
  Vol.~7, 2007, pp. 901--906.

\bibitem{bunescu2007multiple}
R.~C. Bunescu, R.~J. Mooney, Multiple instance learning for sparse positive
  bags, in: Proceedings of the 24th international conference on Machine
  learning, ACM, 2007, pp. 105--112.

\bibitem{zhou2007multi}
Z.-H. Zhou, M.-L. Zhang, Multi-instance multi-label learning with application
  to scene classification, in: Advances in neural information processing
  systems, 2007, pp. 1609--1616.

\bibitem{mangasarian2008multiple}
O.~L. Mangasarian, E.~W. Wild, Multiple instance classification via successive
  linear programming, Journal of Optimization Theory and Applications 137~(3)
  (2008) 555--568.

\bibitem{zhang2009multi}
M.-L. Zhang, Z.-H. Zhou, Multi-instance clustering with applications to
  multi-instance prediction, Applied Intelligence 31~(1) (2009) 47--68.

\bibitem{zhou2009multi}
Z.-H. Zhou, Y.-Y. Sun, Y.-F. Li, Multi-instance learning by treating instances
  as non-iid samples, in: Proceedings of the 26th annual international
  conference on machine learning, ACM, 2009, pp. 1249--1256.

\bibitem{li2009convex}
Y.-F. Li, J.~T. Kwok, I.~W. Tsang, Z.-H. Zhou, A convex method for locating
  regions of interest with multi-instance learning, in: Joint European
  Conference on Machine Learning and Knowledge Discovery in Databases,
  Springer, 2009, pp. 15--30.

\bibitem{vezhnevets2010towards}
A.~Vezhnevets, J.~M. Buhmann, Towards weakly supervised semantic segmentation
  by means of multiple instance and multitask learning, in: Computer Vision and
  Pattern Recognition (CVPR), 2010 IEEE Conference on, IEEE, 2010, pp.
  3249--3256.

\bibitem{foulds2010review}
J.~Foulds, E.~Frank, A review of multi-instance learning assumptions, The
  Knowledge Engineering Review 25~(1) (2010) 1--25.

\bibitem{amores2013multiple}
J.~Amores, Multiple instance classification: Review, taxonomy and comparative
  study, Artificial Intelligence 201 (2013) 81--105.

\bibitem{ray2014multiple}
S.~Ray, S.~Scott, H.~Blockeel, Multiple-instance learning, Encyclopedia of
  Machine Learning and Data Mining (2014) 1--13.

\bibitem{herrera2016multiple}
F.~Herrera, S.~Ventura, R.~Bello, C.~Cornelis, A.~Zafra,
  D.~S{\'a}nchez-Tarrag{\'o}, S.~Vluymans, Multiple instance learning, in:
  Multiple Instance Learning, Springer, 2016, pp. 17--33.

\bibitem{bengio2009learning}
Y.~Bengio, et~al., Learning deep architectures for ai, Foundations and
  trends{\textregistered} in Machine Learning 2~(1) (2009) 1--127.

\bibitem{hinton2006fast}
G.~E. Hinton, S.~Osindero, Y.-W. Teh, A fast learning algorithm for deep belief
  nets, Neural computation 18~(7) (2006) 1527--1554.

\bibitem{goodfellow2016deep}
I.~Goodfellow, Y.~Bengio, A.~Courville, Deep learning, 2016.

\bibitem{girshick2014rich}
R.~Girshick, J.~Donahue, T.~Darrell, J.~Malik, Rich feature hierarchies for
  accurate object detection and semantic segmentation, in: Proceedings of the
  IEEE conference on computer vision and pattern recognition, 2014, pp.
  580--587.

\bibitem{girshick2015fast}
R.~Girshick, Fast r-cnn, in: Proceedings of the IEEE International Conference
  on Computer Vision, 2015, pp. 1440--1448.

\bibitem{ren2015faster}
S.~Ren, K.~He, R.~Girshick, J.~Sun, Faster r-cnn: Towards real-time object
  detection with region proposal networks, in: Advances in neural information
  processing systems, 2015, pp. 91--99.

\bibitem{boser1992training}
B.~E. Boser, I.~M. Guyon, V.~N. Vapnik, A training algorithm for optimal margin
  classifiers, in: Proceedings of the fifth annual workshop on Computational
  learning theory, ACM, 1992, pp. 144--152.

\bibitem{andrew2000introduction}
A.~M. Andrew, An introduction to support vector machines and other kernel-based
  learning methods by nello christianini and john shawe-taylor, cambridge
  university press, cambridge, 2000, xiii+ 189 pp., isbn 0-521-78019-5
  (hbk,{\pounds} 27.50). (2000).

\bibitem{boyd2004convex}
S.~Boyd, L.~Vandenberghe, Convex optimization, Cambridge university press,
  2004.

\bibitem{chapelle2007training}
O.~Chapelle, Training a support vector machine in the primal, Neural
  computation 19~(5) (2007) 1155--1178.

\bibitem{smola1998learning}
A.~J. Smola, B.~Sch{\"o}lkopf, Learning with kernels, GMD-Forschungszentrum
  Informationstechnik, 1998.

\bibitem{FungMangasarian:04}
G.~Fung, O.~L. Mangasarian, Feature selection newton method for support vector
  machine classification, Computational Optimization and Applications 28~(2)
  (2004) 185--202.

\bibitem{Mangasarian:06}
O.~L. Mangasarian, Exact 1-norm support vector machines via unconstrained
  convex differentiable minimization, Journal of Machine Learning Research 7
  (2006) 1517–--1530.

\bibitem{franc2008optimized}
V.~Franc, S.~Sonnenburg, Optimized cutting plane algorithm for support vector
  machines, in: Proceedings of the 25th international conference on Machine
  learning, ACM, 2008, pp. 320--327.

\bibitem{bertsekas1999nonlinear}
D.~P. Bertsekas, Nonlinear programming, Athena scientific Belmont, 1999.

\bibitem{achlioptas2002sampling}
D.~Achlioptas, F.~McSherry, B.~Sch{\"o}lkopf, Sampling techniques for kernel
  methods, Advances in neural information processing systems 1 (2002) 335--342.

\bibitem{tsang2005core}
I.~W. Tsang, J.~T. Kwok, P.-M. Cheung, Core vector machines: Fast svm training
  on very large data sets, Journal of Machine Learning Research 6~(Apr) (2005)
  363--392.

\bibitem{blum2006random}
A.~Blum, Random projection, margins, kernels, and feature-selection, in:
  Subspace, Latent Structure and Feature Selection, Springer, 2006, pp. 52--68.

\bibitem{zanni2006parallel}
L.~Zanni, T.~Serafini, G.~Zanghirati, Parallel software for training large
  scale support vector machines on multiprocessor systems, Journal of Machine
  Learning Research 7~(Jul) (2006) 1467--1492.

\bibitem{loosli2007training}
G.~Loosli, S.~Canu, L.~Bottou, Training invariant support vector machines using
  selective sampling, Large scale kernel machines (2007) 301--320.

\bibitem{bottou2007large}
L.~Bottou, Large-scale kernel machines, MIT press, 2007.

\bibitem{rahimi2008random}
A.~Rahimi, B.~Recht, Random features for large-scale kernel machines, in:
  Advances in neural information processing systems, 2008, pp. 1177--1184.

\bibitem{wang2011trading}
Z.~Wang, N.~Djuric, K.~Crammer, S.~Vucetic, Trading representability for
  scalability: adaptive multi-hyperplane machine for nonlinear classification,
  in: Proceedings of the 17th ACM SIGKDD international conference on Knowledge
  discovery and data mining, ACM, 2011, pp. 24--32.

\bibitem{vedaldi2012efficient}
A.~Vedaldi, A.~Zisserman, Efficient additive kernels via explicit feature maps,
  IEEE transactions on pattern analysis and machine intelligence 34~(3) (2012)
  480--492.

\bibitem{zhang2012scaling}
K.~Zhang, L.~Lan, Z.~Wang, F.~Moerchen, Scaling up kernel svm on limited
  resources: A low-rank linearization approach., in: AISTATS, Vol.~22, 2012,
  pp. 1425--1434.

\bibitem{li2012chebyshev}
F.~Li, G.~Lebanon, C.~Sminchisescu, Chebyshev approximations to the histogram
  $\chi$ 2 kernel, in: Computer Vision and Pattern Recognition (CVPR), 2012
  IEEE Conference on, IEEE, 2012, pp. 2424--2431.

\bibitem{yang2012nystrom}
T.~Yang, Y.-F. Li, M.~Mahdavi, R.~Jin, Z.-H. Zhou, Nystr{\"o}m method vs random
  fourier features: A theoretical and empirical comparison, in: Advances in
  neural information processing systems, 2012, pp. 476--484.

\bibitem{Mu_2014_CVPR}
Y.~Mu, G.~Hua, W.~Fan, S.-F. Chang, Hash-svm: Scalable kernel machines for
  large-scale visual classification, in: The IEEE Conference on Computer Vision
  and Pattern Recognition (CVPR), 2014.

\bibitem{tu2016large}
S.~Tu, R.~Roelofs, S.~Venkataraman, B.~Recht, Large scale kernel learning using
  block coordinate descent, arXiv preprint arXiv:1602.05310.

\bibitem{lu2016large}
J.~Lu, S.~C. Hoi, J.~Wang, P.~Zhao, Z.-Y. Liu, Large scale online kernel
  learning, The Journal of Machine Learning Research 17~(1) (2016) 1613--1655.

\bibitem{ionescu2017large}
C.~Ionescu, A.~Popa, C.~Sminchisescu, Large-scale data-dependent kernel
  approximation, in: Artificial Intelligence and Statistics, 2017, pp. 19--27.

\bibitem{vanwinckelen2016instance}
G.~Vanwinckelen, D.~Fierens, H.~Blockeel, et~al., Instance-level accuracy
  versus bag-level accuracy in multi-instance learning, Data Mining and
  Knowledge Discovery 30~(2) (2016) 313--341.

\end{thebibliography}
\clearpage
\end{document}